\documentclass{article}

\usepackage{microtype}
\usepackage{graphicx}
\usepackage{subfigure}
\usepackage{booktabs} 

\usepackage{hyperref}

\usepackage{amsthm}
\usepackage{amsmath}
\usepackage{amsfonts}
\usepackage{amssymb}
\usepackage{bm}
\usepackage{xcolor}
\usepackage{bbm}
\usepackage{graphicx}
\usepackage{enumitem} 

\theoremstyle{definition}

\newcommand{\mbf}[1]{\mathbf{#1}}
\newcommand{\mcl}[1]{\mathcal{#1}}
\newcommand{\mbb}[1]{\mathbb{#1}}
\newcommand{\argmax}{\text{arg}\!\max}
\newcommand{\argmin}{\text{arg}\!\min}

\newcommand{\stableopt}{\textsc{StableOpt}}
\newcommand{\var}{\textsc{VaR}}
\newcommand{\cvar}{\textsc{CVaR}}
\newcommand{\vy}{V-UCB}
\newcommand{\dx}{\mcl{D}_{\mbf{x}}}
\newcommand{\dz}{\mcl{D}_{\mbf{z}}}
\newcommand{\zlv}{\mbf{z}_{\text{LV}}}
\newcommand{\kg}{$\rho \text{KG}^{apx}$}
\newcommand{\kgfull}{$\rho \text{KG}$}

\usepackage[accepted]{icml2021}

\icmltitlerunning{Value-at-Risk Optimization with Gaussian Processes}

\newtheorem{definition}{Definition}
\newtheorem{theorem}{Theorem}
\newtheorem{corollary}{Corollary}[theorem]
\newtheorem{lemma}{Lemma}
\newtheorem{remark}{Remark}
\newtheorem{example}{Example}

\usepackage{pagecolor}
\begin{document}

\twocolumn[
\icmltitle{Value-at-Risk Optimization with Gaussian Processes}



\icmlsetsymbol{equal}{*}

\begin{icmlauthorlist}
\icmlauthor{Quoc Phong Nguyen}{to}
\icmlauthor{Zhongxiang Dai}{to}
\icmlauthor{Bryan Kian Hsiang Low}{to}
\icmlauthor{Patrick Jaillet}{goo}
\end{icmlauthorlist}

\icmlaffiliation{to}{Department of Computer Science, National University of Singapore, Republic of Singapore}
\icmlaffiliation{goo}{Department of Electrical Engineering and Computer Science, Massachusetts Institute of Technology, USA}

\icmlcorrespondingauthor{Quoc Phong Nguyen}{qphong@comp.nus.edu.sg}

\icmlkeywords{Machine Learning, ICML}

\vskip 0.3in
]



\printAffiliationsAndNotice{} 

\begin{abstract}
\emph{Value-at-risk} ({\var}) is an established measure to assess risks in critical real-world applications with random environmental factors.
This paper presents a novel \emph{{\var} upper confidence bound} (V-UCB) algorithm for maximizing the {\var} of a black-box objective function with the first no-regret guarantee. To realize this, we first derive a confidence bound of {\var} and then prove the existence of values of the environmental random variable (to be selected to achieve no regret) such that the confidence bound of {\var} lies within that of the objective function evaluated at such values.
Our V-UCB algorithm empirically demonstrates state-of-the-art performance in optimizing synthetic benchmark functions, a portfolio optimization problem, and a simulated robot task.
\end{abstract}

\section{Introduction}
Consider the problem of maximizing an expensive-to-compute black-box objective function $f$ that depends on an \emph{optimization variable} $\mbf{x}$ and an \emph{environmental random variable} $\mbf{Z}$.
Due to the randomness in $\mbf{Z}$, the function evaluation $f(\mbf{x},\mbf{Z})$ of $f$ at $\mbf{x}$ is a random variable.
Though for such an objective function $f$, \emph{Bayesian optimization} (BO) can be naturally applied to maximize its expectation $\mbb{E}_{\mbf{Z}}[f(\mbf{x},\mbf{Z})]$ over $\mbf{Z}$ \cite{toscano2018bayesian}, this maximization objective overlooks the \emph{risks} of potentially undesirable function evaluations.
These risks can arise from either (a) the realization of an unknown distribution of $\mbf{Z}$ or (b) the realization of the random $\mbf{Z}$ given that the distribution of $f(\mbf{x},\mbf{Z})$ can be estimated well or that of $\mbf{Z}$ is known.
The issue (a) has been tackled by distributionally robust BO \cite{kirschner2020distributionally,nguyen2020distributionally} which maximizes $\mbb{E}_{\mbf{Z}}[f(\mbf{x},\mbf{Z})]$ under the worst-case realization of the distribution of $\mbf{Z}$.
To resolve the issue (b), 
the risk from the uncertainty in $\mbf{Z}$ can be controlled via the mean-variance optimization framework \cite{iwazaki2020mean}, \emph{value-at-risk} ({\var}), or \emph{conditional value-at-risk} ({\cvar}) \cite{borisk20,torossian2020bayesian}.
The work of \citet{bogunovic2018adversarially} has considered \emph{adversarially robust BO},
where $\mbf{z}$ is controlled by an adversary deterministically.\footnote{We use upper-case letter $\mbf{Z}$ to denote the environmental random variable and lower-case letter $\mbf{z}$ to denote its realization or a (non-random) variable.}  
In this case, the objective is to find $\mbf{x}$ that maximizes the function under the worst-case realization of $\mbf{z}$, i.e., $\argmax_{\mbf{x}} \min_{\mbf{z}} f(\mbf{x},\mbf{z})$.


In this paper, we focus on case (b) where the distribution of 
$\mbf{Z}$ is known (or well-estimated). 
For example, in agriculture, although farmers cannot control the temperature of an outdoor farm, 
its distribution
can be estimated from historical data and controlled in an indoor environment for optimizing the plant yield.
Given the distribution of $\mbf{Z}$, the objective is to control the risk
that the function evaluation $f(\mbf{x},\mbf{z})$, 
for a $\mbf{z}$ sampled from $\mbf{Z}$,
is small.
One popular framework is to control the trade-off between the mean (viewed as reward) and the variance (viewed as risk) of the function evaluation with respect to $\mbf{Z}$ \cite{iwazaki2020mean}.
However, quantifying the risk 
using variance implies indifference between positive and negative deviations from the mean, while people often have asymmetric risk attitudes \cite{goh2012portfolio}. In our problem of maximizing the objective function, it is reasonable to assume that people  are risk-averse towards only the negative deviations 
from the mean, i.e., 
the risk of getting lower function evaluations. 
Thus, it is more appropriate to adopt
risk measures 
with this asymmetric property, such as \emph{value-at-risk} ({\var})
which is a widely adopted risk measure in real-world applications (e.g., banking \cite{basel06}).
Intuitively, the risk 
that the random $f(\mbf{x},\mbf{Z})$ is less than {\var} at level $\alpha \in (0,1)$ does not exceed $\alpha$, 
e.g., by specifying a small value of $\alpha$ as $0.1$, this risk is controlled to be at most $10\%$.
Therefore, to maximize the function $f$ while controlling the risk of undesirable (i.e., small) function evaluations,
we aim to maximize {\var} of the random function $f(\mbf{x},\mbf{Z})$ over $\mbf{x}$.

The recent work of \citet{borisk20} has used BO to maximize {\var} and has achieved state-of-the-art empirical performances.
They have assumed that we are able to select both $\mbf{x}$ and $\mbf{z}$ to query during BO,
which is motivated by fact that physical experiments can usually be studied by 
simulation \cite{williams2000sequential}. 
In the example on agriculture given above, we can control the temperature, light and water ($\mbf{z}$) in a small indoor environment to optimize the amount of fertilizer ($\mbf{x}$), which can then be used in an outdoor environment with random weather factors.
\citet{borisk20} have exploited the ability to select $\mbf{z}$ to model the function $f(\mbf{x},\mbf{z})$ as a GP, which allows them to retain the appealing closed-form posterior belief of the objective function.
To select the queries $\mbf{x}$ and $\mbf{z}$, they have designed a one-step lookahead approach based on the well-known \emph{knowledge gradient} (KG) acquisition function \cite{scott2011correlated}.
However, the one-step lookahead incurs an expensive nested optimization procedure, which is computationally expensive and hence requires approximations.
Besides, the acquisition function can only be approximated using samples of the objective function $f$ from the GP posterior and the environmental random variable $\mbf{Z}$.
While they have analysed the asymptotically unbiased and consistent estimator of the gradients, it is challenging to obtain a guarantee for the convergence of their algorithm.
Another recent work \cite{torossian2020bayesian} has also applied BO to maximize {\var} using an asymmetric Laplace likelihood function and variational approximation of the posterior belief. However, in contrast to \citet{borisk20} and our work, they have focused on a different setting where the realizations of $\mbf{Z}$ are not observed.

In this paper, we adopt the setting of \citet{borisk20} which allows us to choose both $\mbf{x}$ and $\mbf{z}$ to query, and assume that the distribution of $\mbf{Z}$ is known or well-estimated. 
Our contributions include:

\textbf{Firstly}, we propose a novel BO algorithm named \emph{Value-at-risk Upper Confidence Bound} ({\vy}) in Section~\ref{sec:vy}.
Unlike the work of \citet{borisk20}, {\vy} is equipped with a no-regret convergence guarantee and is more computationally efficient.
%
To guide its query selection and facilitate its proof of the no-regret guarantee, the classical GP-UCB algorithm~\cite{srinivas10ucb} constructs a \emph{confidence bound} of the objective function. Similarly, to maximize the {\var} of a random function, we, for the first time to the best of our knowledge, construct a confidence bound of {\var} (Lemma~\ref{lemma:confbound}). 
The resulting confidence bound of {\var} naturally gives rise to a strategy to select $\mbf{x}$.
However, it remains a major challenge to select $\mbf{z}$ to preserve the no-regret convergence of GP-UCB.
To this end, we firstly prove that our algorithm is no-regret as long as we ensure that at the selected $\mbf{z}$, the confidence bound of {\var} \emph{lies within} the confidence bound of the objective function.
Next, we also prove that this query selection strategy is \emph{feasible}, i.e., such values of $\mbf{z}$, referred to as \emph{lacing values} (LV), exist.

\textbf{Secondly}, although our theoretical no-regret property allows the selection of \emph{any} LV, we design a heuristic to select an LV such that it improves our empirical performance over random selection of LV (Section~\ref{subsec:improvedselectz}).
We also discuss the implications when $\mbf{z}$ cannot be selected by BO and is instead randomly sampled by the environment during BO (Remark~\ref{rmk:unknownZ}).
\textbf{Thirdly}, we show that adversarially robust BO \cite{bogunovic2018adversarially} can be cast as a special case of our V-UCB
when the risk level $\alpha$ of {\var} approaches $0$ from the right and the domain of $\mbf{z}$ is the support of $\mbf{Z}$.
In this case, 
adversarially robust BO \cite{bogunovic2018adversarially} selects the same input queries as those selected by V-UCB since the set of LV collapse into the set of minimizers of the lower bound of the objective function (Section~\ref{subsec:vy-stableopt}).
\textbf{Lastly}, we provide practical techniques for
implementing {\vy} with continuous random variable $\mbf{Z}$ (Section~\ref{sec:continuousz}): we (a) introduce \emph{local neural surrogate optimization} with the \emph{pinball loss} to optimize {\var}, and (b) construct an objective function to search for an LV in the continuous support of $\mbf{Z}$.

The performance of our proposed algorithm is empirically demonstrated in optimizing several synthetic benchmark functions, a portfolio optimization problem, and a simulated robot task in Section~\ref{sec:experiments}.

\section{Problem Statement and Background}

Let the objective function be defined as $f: \dx \times \dz \rightarrow \mbb{R}$ where $\dx \subset \mbb{R}^{d_x}$ and $\dz\subset \mbb{R}^{d_z}$ are the bounded domain of the optimization variable $\mbf{x}$ and the support of the environmental random variable $\mbf{Z}$, respectively; $d_x$ and $d_z$ are the dimensions of $\mbf{x}$ and $\mbf{z}$, respectively. The support of $\mbf{Z}$ is defined as the smallest closed subset $\dz$ of $\mbb{R}^{d_z}$ such that $P(\mbf{Z} \in \dz) = 1$.
Let $\mbf{z} \in \dz$ denote a realization of the random variable $\mbf{Z}$. 
Let $f(\mbf{x},\mbf{Z})$ denote a random variable whose randomness comes from $\mbf{Z}$.
%
%
The {\var} of $f(\mbf{x},\mbf{Z})$ at \emph{risk level} $\alpha \in (0,1)$ is defined as:
\begin{equation}
V_{\alpha}(f(\mbf{x},\mbf{Z})) \triangleq \inf \{\omega: P(f(\mbf{x},\mbf{Z}) \le \omega) \ge \alpha\}
\label{eq:var}
\end{equation}
which implies the risk that $f(\mbf{x},\mbf{Z})$ is less than its {\var} at level $\alpha$ does not exceed $\alpha$. 


Our objective is to search for $\mbf{x} \in \dx$ that maximizes $V_\alpha(f(\mbf{x},\mbf{Z}))$ at a user-specified risk level $\alpha \in (0,1)$. 
Intuitively, the goal is find $\mbf{x}$ where the evaluations of the objective function are as large as possible under most realizations of the environmental random variable $\mbf{Z}$ which is characterized by the probability of $1-\alpha$.

The unknown objective function $f(\mbf{x},\mbf{z})$ is modeled with a GP. That is, every finite subset of $\{f(\mbf{x},\mbf{z})\}_{(\mbf{x}, \mbf{z}) \in \dx \times \dz}$ follows a multivariate Gaussian distribution \cite{rasmussen06}.
The GP is fully specified by its \emph{prior} mean and covariance function $k_{(\mbf{x},\mbf{z}), (\mbf{x}',\mbf{z}')} \triangleq \text{cov}[f(\mbf{x},\mbf{z}), f(\mbf{x}',\mbf{z}')]$ for all $\mbf{x}, \mbf{x}'$ in $\dx$ and $\mbf{z}, \mbf{z}'$ in $\dz$. For notational simplicity (and w.l.o.g.), the former is assumed to be zero, while we use the \emph{squared exponential} (SE) kernel as its bounded maximum information gain can be used for later analysis \cite{srinivas10ucb}.

To identify the optimal $\mbf{x}_* \triangleq \argmax_{\mbf{x} \in \dx} V_{\alpha}(f(\mbf{x},\mbf{Z}))$, BO algorithm selects an input query $(\mbf{x}_t,\mbf{z}_t)$ in the $t$-th iteration to obtain a noisy function evaluation $y_{(\mbf{x}_t,\mbf{z}_t)} \triangleq f(\mbf{x}_t,\mbf{z}_t) + \epsilon_t$ where $\epsilon_t \sim \mcl{N}(0,\sigma_n^2)$ are i.i.d. Gaussian noise with variance $\sigma_n^2$. Given noisy observations $\mbf{y}_{\mcl{D}_t} \triangleq (y_{(\mbf{x},\mbf{z})})_{(\mbf{x},\mbf{z}) \in \mcl{D}_t}^\top$ at observed inputs $\mcl{D}_t \triangleq \mcl{D}_{t-1} \cup \{(\mbf{x}_t, \mbf{z}_t)\}$ (and $\mcl{D}_0$ is the initial observed inputs), the GP posterior belief of function evaluation at any input $(\mbf{x},\mbf{z})$ is a Gaussian $p(f(\mbf{x},\mbf{z})|\mbf{y}_{\mcl{D}_t}) \triangleq \mcl{N}( f(\mbf{x},\mbf{z})|\mu_t(\mbf{x},\mbf{z}), \sigma_t^2(\mbf{x},\mbf{z}))$:
\begin{equation}
\begin{array}{r@{}l}
\mu_t(\mbf{x},\mbf{z}) &\triangleq \mbf{K}_{(\mbf{x},\mbf{z}),\mcl{D}_t} \bm{\Lambda}_{\mcl{D}_t\mcl{D}_t} \mbf{y}_{\mcl{D}_t}\ ,\\
\sigma_t^2(\mbf{x},\mbf{z}) &\triangleq k_{(\mbf{x},\mbf{z}),(\mbf{x},\mbf{z})}  - \mbf{K}_{(\mbf{x},\mbf{z}),\mcl{D}_t} \bm{\Lambda}_{\mcl{D}_t\mcl{D}_t} \mbf{K}_{\mcl{D}_t,(\mbf{x},\mbf{z})}
\label{eq:gppost}
\end{array}
\end{equation}
where $\bm{\Lambda}_{\mcl{D}_t\mcl{D}_t} \triangleq \left(
	\mbf{K}_{\mcl{D}_t\mcl{D}_t} + \sigma_n^2 \mbf{I}
\right)^{-1}$, 
$\mbf{K}_{(\mbf{x},\mbf{z}),\mcl{D}_t} \triangleq (k_{(\mbf{x},\mbf{z}),(\mbf{x}',\mbf{z}')})_{(\mbf{x}',\mbf{z}') \in \mcl{D}_t}$, 
$\mbf{K}_{\mcl{D}_t,(\mbf{x},\mbf{z})} \triangleq \mbf{K}_{(\mbf{x},\mbf{z}),\mcl{D}_t}^\top$,
$\mbf{K}_{\mcl{D}_t\mcl{D}_t} \triangleq (k_{(\mbf{x}',\mbf{z}'), (\mbf{x}'',\mbf{z}'')})_{(\mbf{x}',\mbf{z}'), (\mbf{x}'',\mbf{z}'') \in \mcl{D}_t}$,  $\mbf{I}$ is the identity matrix.


\section{BO of {\var}}
\label{sec:vy}

Following the seminal work \cite{srinivas10ucb}, we use the \emph{cumulative regret} as the performance metric to quantify the performance of our BO algorithm. It is defined as $R_T \triangleq \sum_{t=1}^T r(\mbf{x}_t)$ where $r(\mbf{x}_t) \triangleq V_{\alpha}(f(\mbf{x}_*,\mbf{Z})) - V_{\alpha}(f(\mbf{x}_t,\mbf{Z}))$ is the \emph{instantaneous regret} and $\mbf{x}_* \triangleq \argmax_{\mbf{x} \in \mcl{D}_{\mbf{x}}} V_{\alpha}(f(\mbf{x},\mbf{Z}))$. We would like to design a query selection strategy that incurs \emph{no regret}, i.e., $\lim_{T\rightarrow \infty} R_T / T = 0$. 
Furthermore, we have that 
$\min_{t \le T} r(\mbf{x}_t) \le R_T / T$, equivalently, $\max_{t \le T} V_{\alpha}(f(\mbf{x}_t,\mbf{Z})) \ge V_{\alpha}(f(\mbf{x}_*,\mbf{Z})) - R_T / T$.
Thus, $\text{lim}_{T \rightarrow \infty} \max_{t \le T} V_{\alpha}(f(\mbf{x}_t,\mbf{Z})) = V_{\alpha}(f(\mbf{x}_*,\mbf{Z}))$ for a no-regret algorithm.

The proof of the upper bound on the cumulative regret of GP-UCB is based on confidence bounds of the objective function \cite{srinivas10ucb}.
Similarly, in the next section, we start by constructing a confidence bound of $V_{\alpha}(f(\mbf{x},\mbf{Z}))$, which naturally leads to a query selection strategy for $\mbf{x}_t$.
%
\subsection{A Confidence Bound of $V_\alpha(f(\mbf{x},\mbf{Z}))$ and the Query Selection Strategy for $\mbf{x}_t$}
\label{subsec:confboundf}
Firstly, we adopt a confidence bound of the function $f(\mbf{x},\mbf{z})$ from \citet{chowdhury2017kernelized}, which assumes that $f$ belongs to a \emph{reproducing kernel Hilbert space} $\mcl{F}_k(B)$ such that its RKHS norm is bounded $\Vert f\Vert_k \le B$.

\begin{lemma}[\citet{chowdhury2017kernelized}]
Pick $\delta \in (0,1)$ and set 
$\beta_t = (B + \sigma_n \sqrt{2(\gamma_{t-1} + 1 + \log1/ \delta)})^2$. 
Then,
$f(\mbf{x},\mbf{z}) \in I_{t-1}[f(\mbf{x},\mbf{z})] \triangleq [l_{t-1}(\mbf{x},\mbf{z}), u_{t-1}(\mbf{x},\mbf{z})]$
$\forall \mbf{x} \in \mcl{D}_{\mbf{x}}, \mbf{z} \in \mcl{D}_{\mbf{z}}, t \ge 1 $ holds with probability $\ge 1 - \delta$ where
\begin{equation}
\begin{array}{r@{}l}
l_{t-1}(\mbf{x},\mbf{z})\ &\triangleq \mu_{t-1}(\mbf{x},\mbf{z}) - \beta_t^{1/2} \sigma_{t-1}(\mbf{x},\mbf{z})\\
u_{t-1}(\mbf{x},\mbf{z})\ &\triangleq \mu_{t-1}(\mbf{x},\mbf{z}) + \beta_t^{1/2} \sigma_{t-1}(\mbf{x},\mbf{z})\ .
\end{array}\vspace{-1mm}
\label{eq:fbound}
\end{equation}
\label{lemma:ucb51}
\end{lemma}
As the above lemma holds for both finite and continuous $\mcl{D}_x$ and $\mcl{D}_z$, it is 
used
to analyse the regret in both cases. On the other hand, the confidence bound can be adopted to the Bayesian setting by changing only $\beta_t$ following the work of \citet{srinivas10ucb} as noted by \cite{bogunovic2018adversarially}.

Then, we exploit this confidence bound on the function evaluations (Lemma \ref{lemma:ucb51}) to formulate a confidence bound of $V_{\alpha}(f(\mbf{x},\mbf{Z}))$ as follows.
\begin{lemma}
\label{lemma:confbound}
Similar to the definition of $f(\mbf{x},\mbf{Z})$, let $l_{t-1}(\mbf{x}, \mbf{Z})$ and $u_{t-1}(\mbf{x},\mbf{Z})$ denote the random function over $\mbf{x}$ where the randomness comes from the random variable $\mbf{Z}$; $l_{t-1}$ and $u_{t-1}$ are defined in \eqref{eq:fbound}.
%
%
%
Then, $\forall \mbf{x} \in \mcl{D}_{\mbf{x}}$, $t \ge 1$,
\[
\begin{array}{r@{}l}
\displaystyle V_\alpha(f(\mbf{x},\mbf{Z})) \displaystyle &\displaystyle \in I_{t-1}[V_{\alpha}(f(\mbf{x},\mbf{Z}))]\\
	 &\displaystyle \triangleq [V_\alpha(l_{t-1}(\mbf{x},\mbf{Z})), V_\alpha(u_{t-1}(\mbf{x},\mbf{Z}))]
\end{array}
\]
holds with probability $\ge 1 - \delta$ for $\delta$ in Lemma~\ref{lemma:ucb51}, where $V_\alpha(l_{t-1}(\mbf{x},\mbf{Z}))$ and $V_\alpha(u_{t-1}(\mbf{x},\mbf{Z}))$ are defined as \eqref{eq:var}.
\end{lemma}
The proof is in Appendix~\ref{app:proofvconfbound}.
Given the confidence bound $I_{t-1}[V_{\alpha}(f(\mbf{x},\mbf{Z}))] \triangleq [V_\alpha(l_{t-1}(\mbf{x},\mbf{Z})), V_\alpha(u_{t-1}(\mbf{x},\mbf{Z}))$ in Lemma~\ref{lemma:confbound}, we follow the the well-known 
``optimism in the face of uncertainty'' 
principle to select $\mbf{x}_t = \argmax_{\mbf{x} \in \dx} V_\alpha(u_{t-1}(\mbf{x},\mbf{Z}))$. This query selection strategy for $\mbf{x}_t$ leads to an upper bound of $r(\mbf{x}_t)$:
\begin{align}
r(\mbf{x}_t) \le V_\alpha(u_{t-1}(\mbf{x}_t,\mbf{Z})) - V_\alpha(l_{t-1}(\mbf{x}_t,\mbf{Z}))\ \forall t \ge 1
	\label{eq:iregretbound1}
\end{align}
which holds with probability $\ge 1 - \delta$ for $\delta$ in Lemma~\ref{lemma:ucb51}, and is proved in Appendix~\ref{app:iregretbound1}.


As our goal is $\lim_{T\rightarrow \infty} R_T/T = 0$, given the selected query $\mbf{x}_t$, a reasonable query selection strategy of $\mbf{z}_t$ should gather informative observations at $(\mbf{x}_t, \mbf{z}_t)$ that improves the confidence bound $I_{t-1}[V_\alpha (f(\mbf{x}_t,\mbf{Z}))]$ (i.e., $I_{t}[V_\alpha (f(\mbf{x}_t,\mbf{Z}))]$ is a proper subset of $I_{t-1}[V_\alpha (f(\mbf{x}_t,\mbf{Z}))]$ if $I_{t-1}[V_\alpha (f(\mbf{x}_t,\mbf{Z}))] \neq \emptyset$) which can be viewed as the uncertainty of $V_{\alpha}(f(\mbf{x}_t,\mbf{Z}))$.


Assume that there exists
$\mbf{z}_l \in \mcl{D}_{\mbf{z}}$ such that  $l_{t-1}(\mbf{x}_t,\mbf{z}_l) = V_\alpha(l_{t-1}(\mbf{x}_t,\mbf{Z}))$
and
$\mbf{z}_u \in \mcl{D}_{\mbf{z}}$ such that  $u_{t-1}(\mbf{x}_t,\mbf{z}_u) = V_\alpha(u_{t-1}(\mbf{x}_t,\mbf{Z}))$.
Lemma~\ref{lemma:confbound} implies that $V_\alpha(f(\mbf{x}_t,\mbf{Z})) \in I_{t-1}[V_\alpha (f(\mbf{x}_t,\mbf{Z}))] = [l_{t-1}(\mbf{x}_t,\mbf{z}_l), u_{t-1}(\mbf{x}_t,\mbf{z}_u)]$ with high probability. 
Hence, we may 
na\"ively
want to query for observations at $(\mbf{x}_t, \mbf{z}_l)$ and $(\mbf{x}_t,\mbf{z}_u)$ to reduce $I_{t-1}[V_\alpha (f(\mbf{x}_t,\mbf{Z}))]$. 
However, these observations may not always reduce $I_{t-1}[V_\alpha (f(\mbf{x}_t,\mbf{Z}))]$. 
The insight is that $I_{t-1}[V_\alpha (f(\mbf{x}_t,\mbf{Z}))]$ changes (i.e., shrinks) when either of its boundary values (i.e., $l_{t-1}(\mbf{x}_t,\mbf{z}_l)$ or $u_{t-1}(\mbf{x}_t,\mbf{z}_u)$) changes. 
Consider $u_{t-1}(\mbf{x}_t,\mbf{z}_u)$ and finite $\dz$ as an example, since $u_{t-1}(\mbf{x}_t,\mbf{z}_u) = V_\alpha(u_{t-1}(\mbf{x}_t,\mbf{Z}))$, 
a natural cause for the change in $u_{t-1}(\mbf{x}_t,\mbf{z}_u)$ is when $\mbf{z}_u$ changes. 
This happens if there exists $\mbf{z}' \neq \mbf{z}_u$ such that the \emph{ordering} of $u_{t-1}(\mbf{x}_t,\mbf{z}')$ relative to $u_{t-1}(\mbf{x}_t,\mbf{z}_u)$ changes given more observations. Thus, observations that are capable of reducing $I_{t-1}[V_\alpha (f(\mbf{x}_t,\mbf{Z}))]$ should be able to \emph{change the relative ordering} in this case.
We construct the following counterexample where observations at $\mbf{z}_u$ (and $\mbf{z}_l$) are not able to change the relative ordering, so they do not reduce $I_{t-1}[V_\alpha (f(\mbf{x}_t,\mbf{Z}))]$.

\begin{example}
\label{example:counter}
This example is described by Fig.~\ref{fig:orderinguncertainty}. 
We reduce notational clutter by removing $\mbf{x}_t$ and $t$ since they are fixed in this example, i.e., we use $f(\mbf{z})$, $f(\mbf{Z})$, and $l(\mbf{z})$ to denote $f(\mbf{x}_t, \mbf{z})$, $f(\mbf{x}_t, \mbf{Z})$, and $l_{t-1}(\mbf{x}_t, \mbf{z})$ respectively. 
We condition on the event $f(\mbf{z}) \in I[f(\mbf{z})] \triangleq [l(\mbf{z}), u(\mbf{z})]$ for all $\mbf{z} \in \dz$ which occurs with probability $\ge 1 - \delta$ in Lemma~\ref{lemma:ucb51}. In this example, $\mbf{z}_l = \mbf{z}_1$ and $l(\mbf{z}_1) = u(\mbf{z}_1)$, so there is no uncertainty in $f(\mbf{z}_l) = f(\mbf{z}_1)$. Similarly, there is no uncertainty in $f(\mbf{z}_u) = f(\mbf{z}_2)$. Thus, new observations at $\mbf{z}_l$ and $\mbf{z}_u$ change neither $l(\mbf{z}_l)$ nor $u(\mbf{z}_u)$, so these observations do not reduce the confidence bound $I[V_{\alpha=0.4}(f(\mbf{Z}))] = [l(\mbf{z}_l), u(\mbf{z}_u)]$ (plotted as the double-headed arrow in Fig.~\ref{fig:orderinguncertainty}b).
In fact, to reduce $I[V_{\alpha=0.4}(f(\mbf{Z}))]$, we should gather new observations at $\mbf{z}_0$ which potentially change the ordering of $u(\mbf{z}_0)$ relative to $u(\mbf{z}_2)$ (which is $u(\mbf{z}_u)$ without new observations). For example, after 
getting
new observations at $\mbf{z}_0$, if $u(\mbf{z}_0)$ is improved to be in the white region between A and B ($u(\mbf{z}_0) > u(\mbf{z}_2)$ in Fig.~\ref{fig:orderinguncertainty}b changes to $u(\mbf{z}_0) < u(\mbf{z}_2)$ in Fig.~\ref{fig:orderinguncertainty}c), then $I[V_{\alpha=0.4}(f(\mbf{Z}))]$ is reduced to $[l(\mbf{z}_1), u(\mbf{z}_0)]$ 
because now $\mbf{z}_u=\mbf{z}_0$.
Thus, as the confidence bound $I[f(\mbf{z}_0)]$ is shortened with more and more observations at $\mbf{z}_0$, the confidence bound $I[V_{\alpha=0.4}(f(\mbf{Z}))]$ reduces  (the white region in Fig.~\ref{fig:orderinguncertainty} is `laced up').
\end{example}
%
%
\begin{figure*}[h!]
\centering
\begin{tabular}{@{}ccc@{}}
\includegraphics[width=0.25\textwidth]{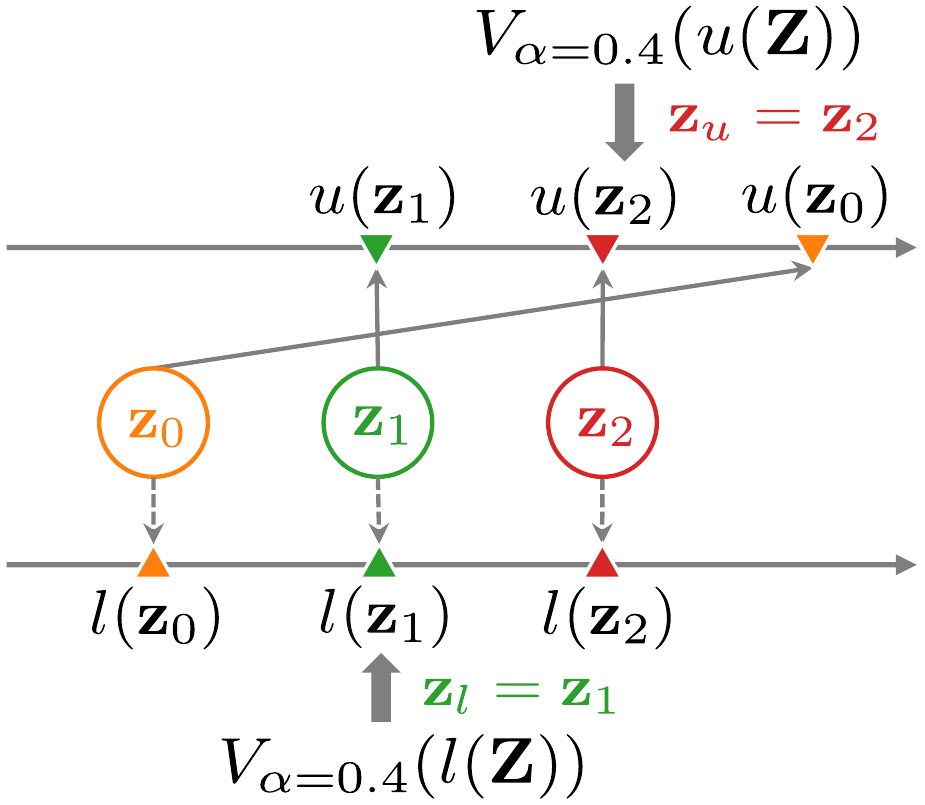}
&
\includegraphics[width=0.25\textwidth]{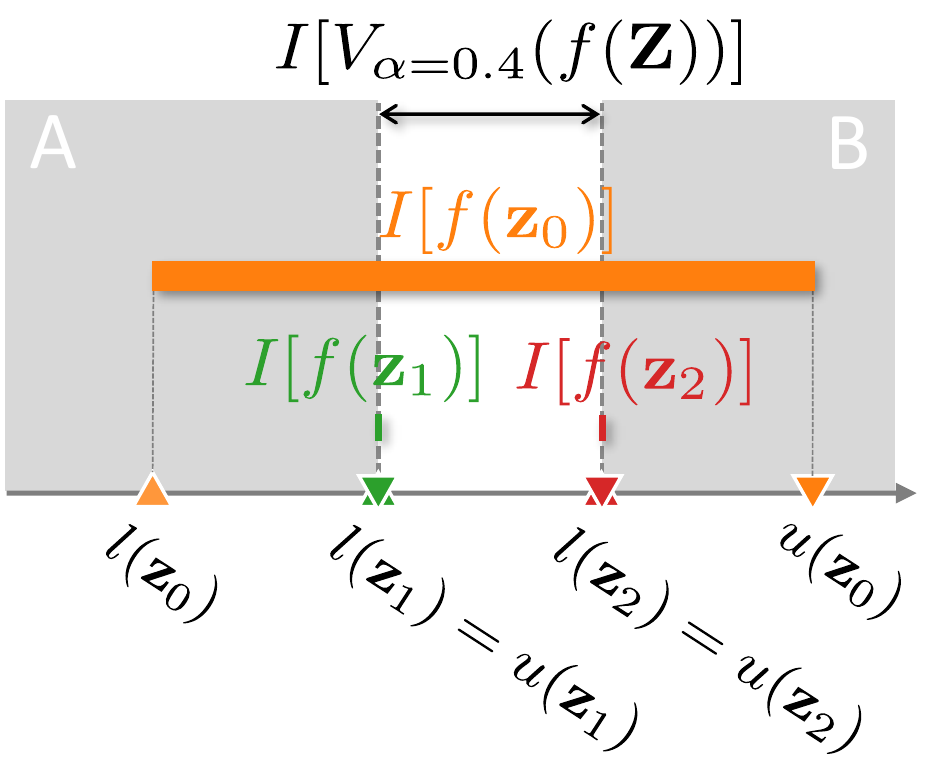}
&
\includegraphics[width=0.25\textwidth]{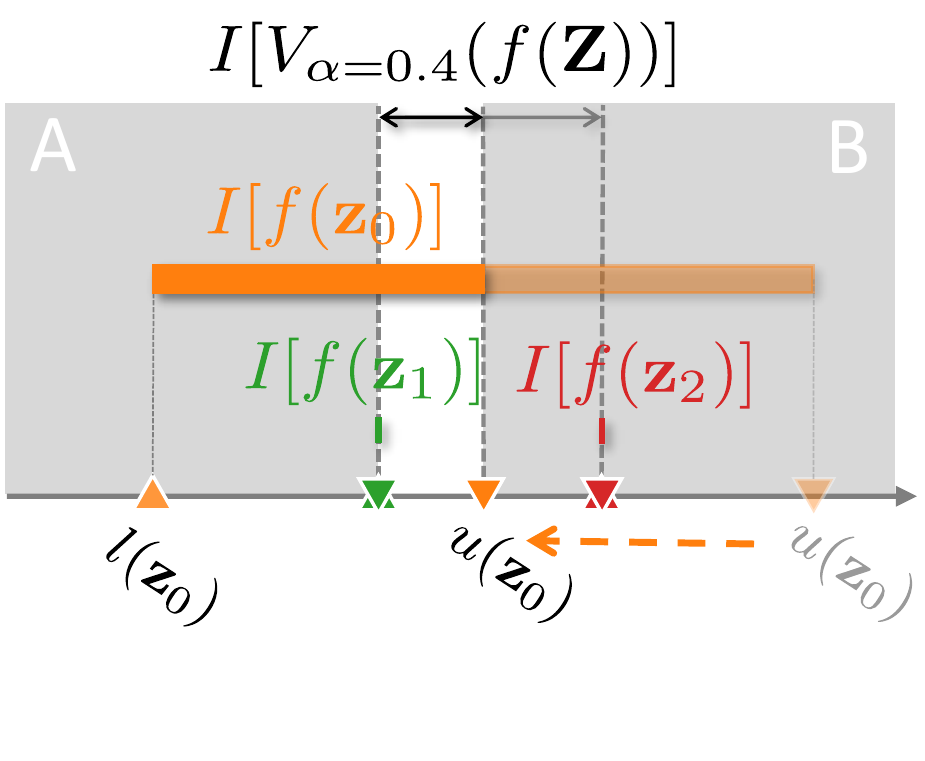}
\\
(a) & (b) & (c)
\end{tabular}
\caption{
A counterexample against selecting $\mbf{z}_u$ and $\mbf{z}_l$ as input queries.
Here $\mbf{z}$ follows a discrete uniform distribution over $\dz \triangleq \{\mbf{z}_0, \mbf{z}_1, \mbf{z}_2\}$.
(a) shows the mappings of $\mbf{z}$ to the upper bound $u(\mbf{z})$ and lower bound $l(\mbf{z})$.
The {\var} at $\alpha=0.4$ of $u(\mbf{Z})$ and $l(\mbf{Z})$ are $u(\mbf{z}_2)$ and $l(\mbf{z}_1)$, respectively, i.e., $\mbf{z}_u = \mbf{z}_2$ and $\mbf{z}_l = \mbf{z}_1$.
(b) shows the values of $l(\mbf{z})$ and $u(\mbf{z})$ for all $\mbf{z}$ on the same axis, as well as the confidence bounds of $f(\mbf{z})$ and $V_{\alpha}(f(\mbf{Z}))$. The gray areas A and B indicate the intervals of values $\omega \in \mbb{R}$ where $\omega \le l(\mbf{z}_l) = l(\mbf{z}_1)$ and $\omega \ge u(\mbf{z}_u) = u(\mbf{z}_2)$, respectively. 
(c) shows a hypothetical scenario when $I[f(\mbf{z}_0)]$ is shortened with more observations at $\mbf{z}_0$.
}
\label{fig:orderinguncertainty}
\end{figure*}

In the next section, we define a property of $\mbf{z}_0$ in Example~\ref{example:counter} 
and prove the existence of $\mbf{z}$'s with this property.
Then, we prove that along with the optimistic selection of $\mbf{x}_t$, the selection of $\mbf{z}_t$ such that it satisfies this property leads to a no-regret algorithm.

\subsection{Lacing Value (LV) and the Query Selection Strategy for $\mbf{z}_t$}
\label{subsec:lv}

We note that in Example~\ref{example:counter}, as long as the confidence bound of the function evaluation at $\mbf{z}_0$ contains 
the confidence bound
of {\var}, observations at $\mbf{z}_0$ can 
reduce the confidence bound of {\var}. We name the values of $\mbf{z}$ satisfying this property as \emph{lacing values} (LV):

\begin{definition}[Lacing values]
\label{definition:lv}
\emph{Lacing values} (LV) with respect to $\mbf{x} \in \dx$ and $t \ge 1$ are $\zlv \in \dz$ that satisfies $l_{t-1}(\mbf{x},\zlv) \le V_{\alpha}(l_{t-1}(\mbf{x}, \mbf{Z})) \le V_{\alpha}(u_{t-1}(\mbf{x}, \mbf{Z})) \le u_{t-1}(\mbf{x},\zlv)$, equivalently,
$I_{t-1}[V_{\alpha}(f(\mbf{x},\mbf{Z}))] \subset [l_{t-1}(\mbf{x},\zlv), u_{t-1}(\mbf{x},\zlv)]\ .$
\end{definition}

Recall that the support $\dz$ of $\mbf{Z}$ is defined as the smallest closed subset $\dz$ of $\mbb{R}^{d_z}$ such that $P(\mbf{Z} \in \dz) = 1$ (e.g., $\dz$ is a finite subset of $\mbb{R}^{d_z}$ and $\dz = \mbb{R}^{d_z}$). 
The following theorem guarantees the existence of lacing values and is proved in Appendix~\ref{app:prooflv}.

\begin{theorem}[Existence of lacing values]
\label{theorem:lv}
$\forall \alpha \in (0,1)$, $\forall \mbf{x} \in \dx$, $\forall t \ge 1$, there exists a lacing value in $\dz$ with respect to $\mbf{x}$ and $t$.
\end{theorem}

\begin{corollary}
\label{corollary:loclv}
Lacing values with respect to $\mbf{x} \in \dx$ and $t \ge 1$ are in $\mcl{Z}_l^{\le} \cap \mcl{Z}_u^{\ge}$ where $\mcl{Z}_l^{\le} \triangleq \{\mbf{z} \in \mcl{D}_{\mbf{z}}: l_{t-1}(\mbf{x},\mbf{z}) \le V_{\alpha}(l_{t-1}(\mbf{x}, \mbf{Z}))\}$ and $\mcl{Z}_u^{\ge} \triangleq \{\mbf{z} \in \mcl{D}_{\mbf{z}}: u_{t-1}(\mbf{x},\mbf{z}) \ge V_{\alpha}(u_{t-1}(\mbf{x}, \mbf{Z}))\}$.
\end{corollary}

As a special case, when $\mbf{z}_l = \mbf{z}_u$,  $I_{t-1}[V_\alpha(f(\mbf{x},\mbf{Z}))] = I_{t-1}[f(\mbf{x},\mbf{z}_l)]$ which means $\mbf{z}_l = \mbf{z}_u$ is an LV.
%
Based on Theorem~\ref{theorem:lv}, we can always select $\mbf{z}_t$ as an LV w.r.t $\mbf{x}_t$ defined in Definition~\ref{definition:lv}.
This strategy for the selection of $\mbf{z}_t$, together with the selection of $\mbf{x}_t = \argmax_{\mbf{x} \in \dx} V_\alpha(u_{t-1}(\mbf{x},\mbf{Z}))$ (Section \ref{subsec:confboundf}), completes our algorithm: \emph{{\var} Upper Confidence Bound} ({\vy}) (Algorithm~\ref{alg:v-ucb}).

\textbf{Upper Bound on Regret.}
As a result of the selection strategies for $\mbf{x}_t$ and $\mbf{z}_t$, our V-UCB algorithm enjoys the following upper bound on its instantaneous regret (proven in Appendix~\ref{app:iregretbound2}):
\begin{lemma}
\label{lemma:iregretbound2}
By selecting $\mbf{x}_t$ as a maximizer of $V_\alpha(u_{t-1}(\mbf{x},\mbf{Z}))$ and selecting $\mbf{z}_t$ as an LV w.r.t $\mbf{x}_t$, the instantaneous regret is upper-bounded by:
\begin{equation*}
r(\mbf{x}_t) \le 2 \beta_t^{1/2} \sigma_{t-1}(\mbf{x}_t, \mbf{z}_t)\ \forall t \ge 1
\end{equation*}
with probability $\ge 1 - \delta$ for $\delta$ in Lemma~\ref{lemma:ucb51}.
\end{lemma}
Lemma~\ref{lemma:iregretbound2}, together with Lemma 5.4 from \citet{srinivas10ucb}, implies that the cumulative regret of our algorithm is bounded (Appendix~\ref{app:rtbound}): $R_T \le \sqrt{C_1 T \beta_T \gamma_T}$ 
where $C_1 \triangleq 8/\log(1 + \sigma_n^{-2})$, and $\gamma_T$ is the maximum information gain about $f$ that can be obtained from any set of $T$ observations.
\citet{srinivas10ucb} have analyzed $\gamma_T$ for several commonly used kernels such as SE and Mat\'ern kernels, and have shown that for these kernels, the upper bound on $R_T$ grows sub-linearly.
This implies that our algorithm is \emph{no-regret} because $\lim_{T\rightarrow \infty} R_T / T = 0$.


\begin{algorithm}[tb]
   \caption{The {\vy} Algorithm}
\begin{algorithmic}[1]
   \STATE {\bfseries Input:} $\mcl{D}_{\mbf{x}}$, $\mcl{D}_{\mbf{z}}$, prior $\mu_0=0, \sigma_0, k$
   \FOR{$i=1,2,\dots$}
   \STATE Select $\mbf{x}_t = \argmax_{\mbf{x} \in \mcl{D}_{\mbf{x}}} V_{\alpha}(u_{t-1}(\mbf{x},\mbf{Z}))$
   \STATE Select $\mbf{z}_t$ as a \emph{lacing value} w.r.t. $\mbf{x}_t$ (Definition~\ref{definition:lv})
   \STATE Obtain observation $y_t \triangleq f(\mbf{x}_t,\mbf{z}_t) + \epsilon_t$
   \STATE Update the GP posterior belief to obtain $\mu_t$ and $\sigma_t$
   \ENDFOR
\end{algorithmic}
\label{alg:v-ucb}
\end{algorithm}

Inspired by \citet{bogunovic2018adversarially}, at the {$T$-th} iteration of {\vy}, we can recommend $\mbf{x}_{t_*(T)}$ as an estimate of the maximizer $\mbf{x}_*$ of $V_\alpha(f(\mbf{x},\mbf{Z}))$, where $t_*(T) \triangleq  \argmax_{t \in \{1,\dots,T\}} V_{\alpha}(l_{t-1}(\mbf{x}_t,\mbf{Z}))$. Then, the instantaneous regret $r(\mbf{x}_{t_*(T)})$ is upper-bounded by $\sqrt{C_1 \beta_T \gamma_T / T}$ with high probability as we show in Appendix~\ref{app:recommendxbound}.
In our experiments in Section~\ref{sec:experiments}, 
we recommend $\argmax_{\mbf{x} \in \mcl{D}_T} V_{\alpha}(\mu_{t-1}(\mbf{x}, \mbf{Z}))$ (where $\mu_{t-1}(\mbf{x}, \mbf{Z})$ is a random function defined in the same manner as $f(\mbf{x},\mbf{Z})$) as an estimate of $\mbf{x}_*$ due to its empirical convergence.

\textbf{Computational Complexity.}
To compare our computational complexity
with that of the {\kgfull} algorithm from \citet{borisk20}, we exclude the common part of training the GP model (line 6) and assume that $\dz$ is finite. 
Then, the time complexity of {\vy} is dominated by that of the selection of $\mbf{x}_t$ (line 3) which includes the time complexity $\mcl{O}(|\dz| |\mcl{D}_{t-1}|^2)$ for the GP prediction at $\{\mbf{x}\} \times \dz$, and $\mcl{O}(|\dz|\log|\dz|)$ for the sorting of $u_{t-1}(\mbf{x},\dz)$ and searching of {\var}.
Hence, our overall complexity is $\mcl{O}(n |\dz|\ (|\mcl{D}_{t-1}|^2 + \log |\dz|))$, where $n$ is the number of iterations to maximize $V_\alpha(u_{t-1}(\mbf{x},\mbf{Z}))$ (line 3). 
Therefore, our V-UCB is more computationally efficient than {\kgfull} and its variant with approximation {\kg}, whose complexities are $\mcl{O}(n_{\text{out}} n_{\text{in}} K |\dz|\ (|\mcl{D}_{t-1}|^2 + |\dz| |\mcl{D}_{t-1}| + |\dz|^2 + M |\dz|))$ of {\kgfull} and $\mcl{O}(n_{\text{out}} |\mcl{D}_{t-1}| K |\dz|\ (|\mcl{D}_{t-1}|^2 + |\dz| |\mcl{D}_{t-1}| + |\dz|^2 + M |\dz|))$, respectively.\footnote{$n_{\text{out}}$ and $n_{\text{in}}$ are the numbers of iterations for the outer and inner optimization respectively, $K$ is the number of fantasy GP models required for their one-step lookahead, and $M$ is the number of functions sampled from the GP posterior \cite{borisk20}.}

\subsection{On the Selection of $\mbf{z}_t$}
\label{subsec:improvedselectz}

Although Algorithm~\ref{alg:v-ucb} is guaranteed to be no-regret with any choice of LV as $\mbf{z}_t$, we would like to select the LV that can reduce a large amount of the uncertainty of $V_{\alpha}(f(\mbf{x}_t,\mbf{Z}))$. 
However, relying on the information gain measure or the knowledge gradient method often incurs the expensive one-step lookahead.
Therefore, we use a simple heuristic by choosing the LV $\zlv$ with the maximum probability mass (or probability density if $\mbf{Z}$ is continuous) of $\zlv$.
We motivate this heuristic using an example with $\alpha = 0.2$, i.e., $V_{\alpha=0.2}(f(\mbf{x}_t,\mbf{Z})) = \inf \{\omega: P(f(\mbf{x}_t,\mbf{Z}) \le \omega) \ge 0.2\}$.
Suppose $\dz$ is finite and there are $2$ LV's $\mbf{z}_0$ and $\mbf{z}_1$ with $P(\mbf{z}_0) \ge 0.2$ and $P(\mbf{z}_1) = 0.01$. 
Then, because $P(f(\mbf{x}_t,\mbf{Z}) \le f(\mbf{x}_t,\mbf{z}_0)) \ge P(\mbf{z}_0) \ge 0.2$, it follows that $V_{\alpha=0.2}(f(\mbf{x}_t,\mbf{Z})) \le f(\mbf{x}_t,\mbf{z}_0)$, i.e., querying $\mbf{z}_0$ at $\mbf{x}_t$ gives us information about an explicit upper bound on $V_{\alpha=0.2}(f(\mbf{x}_t,\mbf{Z}))$ to reduce its uncertainty.
In contrast, this cannot be achieved by querying $\mbf{z}_1$ due to its low probability mass.
This simple heuristic can also be implemented when $\mbf{Z}$ is a continuous random variable which we will introduce in Section~\ref{sec:continuousz}.
%
\begin{remark}
\label{rmk:unknownZ}
Although we assume that we can select both $\mbf{x}_t$ and $\mbf{z}_t$ during our
algorithm, Corollary~\ref{corollary:loclv} also gives us some insights about the scenario where we cannot select $\mbf{z}_t$.
In this case, in each iteration $t$, we select $\mbf{x}_t$ while $\mbf{z}_t$ is randomly sampled by the environment following the distribution of $\mbf{Z}$. 
Next, we observe both $\mbf{z}_t$ and $y_{(\mbf{x}_t,\mbf{z}_t)}$ and then update the GP posterior belief of $f$.
Of note, Corollary~\ref{corollary:loclv} has shown that all LV lie in the set $\mcl{Z}_l^{\le} \cap \mcl{Z}_u^{\ge}$. 
However, the probability of this set is usually small, because $P(\mbf{Z} \in \mcl{Z}_l^{\le} \cap \mcl{Z}_u^{\ge}) \le P(\mbf{Z} \in \mcl{Z}_l^{\le}) \le \alpha$ and small values of $\alpha$ are often used by real-world applications to manage risks.
Thus, the probability that the sampled $\mbf{z}_t$ is an LV is small. 
As a result, we suggest sampling a large number of $\mbf{z}_t$'s from the environment to increase the chance that an LV is sampled.
On the other hand, the small probability of sampling an LV 
motivates the need for us to select $\mbf{z}_t$.
\end{remark}

\subsection{{\vy} Approaches {\stableopt} as $\alpha \rightarrow 0^+$}
\label{subsec:vy-stableopt}

Recall that the objective of adversarially robust BO is to find $\mbf{x} \in \dx$ that maximizes
$\min_{\mbf{z} \in \mcl{D}_{\mbf{z}}} f(\mbf{x}, \mbf{z})$ \cite{bogunovic2018adversarially} by iteratively specifying input query $(\mbf{x}_t, \mbf{z}_t)$ to collect noisy observations $y_{\mbf{x}_t,\mbf{z}_t}$. 
It is different from BO of {\var} since its $\mbf{z}$ is not random but selected by an adversary who aims to minimize the function evaluation. 
The work of \citet{bogunovic2018adversarially} has proposed a no-regret algorithm for this setting named {\stableopt}, which selects
%
\begin{equation}
\begin{array}{r@{}l}
\mbf{x}_t &= \argmax_{\mbf{x} \in \mcl{D}_x} \min_{\mbf{z} \in \mcl{D}_z} u_{t-1}(\mbf{x},\mbf{z})\ ,\\
\mbf{z}_t &= \argmin_{\mbf{z} \in \mcl{D}_z} l_{t-1}(\mbf{x}_t,\mbf{z})
\end{array}
\label{eq:stableopt}
\end{equation}
where $u_{t-1}$ and $l_{t-1}$ are defined in \eqref{eq:fbound}.

At first glance, BO of {\var} and adversarially robust BO are seemingly different problems because $\mbf{Z}$ is a random variable in the former but not in the latter. 
However, based on our key observation on the connection between the minimum value of a continuous function $h(\mbf{w})$ and the {\var} of the random variable $h(\mbf{W})$ in the following theorem, these two problems and their solutions are connected as shown in Corollary~\ref{corollary:alpha0lv}, and~\ref{corollary:stableopt}.

\begin{theorem}
\label{theorem:0plus}
Let $\mbf{W}$ be a random variable with the support $\mcl{D}_w \subset \mbb{R}^{d_w}$ and dimension $d_w$.
Let $h$ be a continuous function mapping from $\mbf{w} \in \mcl{D}_w$ to $\mbb{R}$. Then, $h(\mbf{W})$ denotes the random variable whose realization is the function $h$ evaluation at a realization $\mbf{w}$ of $\mbf{W}$. Suppose $h(\mbf{w})$ has a minimizer $\mbf{w}_{\min} \in \mcl{D}_w$, then $\lim_{\alpha \rightarrow 0^+} V_{\alpha}(h(\mbf{W})) = h(\mbf{w}_{\min})\ .$
\end{theorem}

\begin{corollary}
\label{corollary:alpha0lv}
Adversarially robust BO which finds $\argmax_{\mbf{x}} \min_{\mbf{z}} f(\mbf{x},\mbf{z})$ can be cast as BO of {\var} by letting (a) $\alpha$ approach $0$ from the right and (b) $\dz$ be the support of the environmental random variable $\mbf{Z}$, i.e., $\argmax_{\mbf{x}} \lim_{\alpha \rightarrow 0^+} V_\alpha(f(\mbf{x},\mbf{Z}))$.
\end{corollary}

Interestingly, from Theorem~\ref{theorem:0plus}, we observe that $\mcl{Z}_l^\le$ in Corollary~\ref{corollary:loclv} 
approaches 
the set of minimizers  $\argmin_{\mbf{z} \in \dz} l_{t-1}(\mbf{x}_t,\mbf{z})$ as $\alpha \rightarrow 0^+$. 
Corollary \ref{corollary:stableopt} below shows that LV w.r.t $\mbf{x}_t$ becomes a minimizer of $l_{t-1}(\mbf{x}_t,\mbf{z})$ which is same as the selected $\mbf{z}_t$ of {\stableopt} in \eqref{eq:stableopt}.

\begin{corollary}
\label{corollary:stableopt}
The {\stableopt} solution to adversarially robust BO selects the same input query as that selected by the {\vy} solution to the corresponding BO of {\var} in Corollary~\ref{corollary:alpha0lv}.
\end{corollary}

The proof of Theorem~\ref{theorem:0plus} and its two corollaries are shown in Appendix~\ref{app:0plus}.
We note that {\vy} is also applicable to the optimization of $V_{\alpha}(f(\mbf{x},\mbf{Z}))$ where the distribution of $\mbf{Z}$ is a conditional distribution given $\mbf{x}$. For example,
in robotics, if there exists noise/perturbation in the control, an optimization problem of interest is
$V_{\alpha}(f(\mbf{x} + \bm{\xi}(\mbf{x})))$ where $\bm{\xi}(\mbf{x})$ is the random perturbation that depends on $\mbf{x}$.



\subsection{Implementation of V-UCB with Continuous Random Variable $\mbf{Z}$}
\label{sec:continuousz}

The V-UCB algorithm involves two steps: selecting $\mbf{x}_t = \argmax_{\mbf{x} \in \dx} V_{\alpha}(u_{t-1}(\mbf{x},\mbf{Z}))$ and selecting $\mbf{z}_t$ as the LV $\zlv$ with the largest probability mass (or probability density). When $|\dz|$ is finite, given $\mbf{x}$,  $V_{\alpha}(u_{t-1}(\mbf{x},\mbf{Z}))$ can be computed exactly. 
The gradient of $V_{\alpha}(u_{t-1}(\mbf{x},\mbf{Z}))$ with respect to $\mbf{x}$ can be obtained 
easily (e.g., using
automatic differentiation provided in the Tensorflow library \cite{tensorflow2015-whitepaper}) 
to aid the selection of $\mbf{x}_t$.
In this case, the latter step can also be performed by constructing the set of all LV (checking the condition in the Definition~\ref{definition:lv} for all $\mbf{z} \in \dz$) and choosing the LV $\zlv$ with the largest probability mass.

{\bf{Estimation of {\var}.}} When $\mbf{Z}$ is a continuous random variable, estimating {\var} by an ordered set of samples (e.g., in \citet{borisk20}) may require a prohibitively large number of samples, especially for small values of $\alpha$. Thus, we employ the following popular pinball (or tilted absolute value) function in quantile regression \cite{koenker1978regression} to estimate {\var} as a lower $\alpha$-quantile:
\begin{equation*}
\rho_{\alpha}(w) \triangleq \begin{cases}
    \alpha w &\text{if } w \ge 0\ ,\\
    (\alpha - 1) w &\text{if } w < 0
\end{cases}
\end{equation*}
where $w \in \mbb{R}$. In particular, to estimate $V_{\alpha}(u_{t-1}(\mbf{x},\mbf{Z}))$ as $\nu \in \mbb{R}$, we find $\nu$ that minimizes:
\begin{equation}
    \mbb{E}_{\mbf{z} \sim p(\mbf{Z})}[\rho_{\alpha}(u_{t-1}(\mbf{x},\mbf{z}) - \nu)]\ .
    \label{eq:lossfindvar}
\end{equation}
A well-known example is when $\alpha = 0.5$ and $\rho_{\alpha}$ is the absolute value function, then the optimal $\nu$ is the median. 
The loss in \eqref{eq:lossfindvar} can be optimized using stochastic gradient descent with a random batch of samples of $\mbf{Z}$ at each optimization iteration. 

{\bf{Maximization of $V_{\alpha}(u_{t-1}(\mbf{x},\mbf{Z}))$.}} Unfortunately, to maximize $V_{\alpha}(u_{t-1}(\mbf{x},\mbf{Z}))$ over $\mbf{x} \in \dx$, there is no gradient of $V_{\alpha}(u_{t-1}(\mbf{x},\mbf{Z}))$ with respect to $\mbf{x}$ under the above approach. This situation resembles BO where there is no gradient information, but only noisy observations at input queries. Unlike BO, the observation (samples of $u_{t-1}(\mbf{x},\mbf{Z})$ at $\mbf{x}$) is not costly. Therefore, we propose the \emph{local neural surrogate optimization} (LNSO) algorithm to find $\argmax_{\mbf{x} \in \dx} V_{\alpha}(u_{t-1}(\mbf{x},\mbf{Z}))$ which is visualized in Fig.~\ref{fig:lnso}. Suppose the optimization is initialized at $\mbf{x} = \mbf{x}^{(0)}$, instead of maximizing $V_{\alpha}(u_{t-1}(\mbf{x},\mbf{Z}))$ (whose gradient w.r.t. $\mbf{x}$ is unknown), we maximize a surrogate function $g(\mbf{x},\bm{\theta}^{(0)})$ (modeled by a neural network) that approximates $V_{\alpha}(u_{t-1}(\mbf{x},\mbf{Z}))$ well in a local region of $\mbf{x}^{(0)}$, e.g., a ball $\mcl{B}(\mbf{x}^{(0)}, r)$ of radius $r$ in Fig.~\ref{fig:lnso}. 
We obtain such parameters $\bm{\theta}^{(0)}$ by minimizing the following loss function:
\begin{equation}
\begin{array}{l}
\mcl{L}_g(\bm{\theta},\mbf{x}^{(0)})\\ 
\ \displaystyle\triangleq \mbb{E}_{\mbf{x} \in \mcl{B}(\mbf{x}^{(0)}, r)} \mbb{E}_{\mbf{Z} \sim p(\mbf{Z})} \left[ \rho_\alpha(u_{t-1}(\mbf{x},\mbf{z}) - g(\mbf{x};\bm{\theta}))\right]
\end{array}
\label{eq:losslnso}
\end{equation}
where the expectation $\mbb{E}_{\mbf{x} \in \mcl{B}(\mbf{x}^{(0)}, r)}$ is taken over a uniformly distributed $\mbf{X}$ in $\mcl{B}(\mbf{x}^{(0)}, r)$. Minimizing \eqref{eq:losslnso} can be performed with stochastic gradient descent.
If maximizing $g(\mbf{x},\bm{\theta}^{(0)})$ leads to a value $\mbf{x}^{(i)} \notin \mcl{B}(\mbf{x}^{(0)}, r)$ (Fig.~\ref{fig:lnso}), we update the local region to be centered at $\mbf{x}^{(i)}$ ($\mcl{B}(\mbf{x}^{(i)}, r)$) and find  $\bm{\theta}^{(i)} = \argmin_{\bm{\theta}} \mcl{L}_g(\bm{\theta}, \mbf{x}^{(i)})$ such that $g(\mbf{x},\bm{\theta}^{(i)})$ approximates $V_{\alpha}(u_{t-1}(\mbf{x},\mbf{Z}))$ well $\forall \mbf{x} \in \mcl{B}(\mbf{x}^{(i)},r)$. Then, $\mbf{x}^{(i)}$ is updated by maximizing $g(\mbf{x},\bm{\theta}^{(i)})$ for $\mbf{x} \in \mcl{B}(\mbf{x}^{(i)},r)$. The complete algorithm is described in Appendix~\ref{app:lnso}.

We prefer a small value of $r$ so that the ball $\mcl{B}$ is small. In such case, $V_\alpha(u_{t-1}(\mbf{x},\mbf{Z}))$ for $\mbf{x} \in \mcl{B}$ can be estimated well with a small neural network $g(\mbf{x},\bm{\theta})$ whose training requires a small number of iterations. 

\begin{figure}
    \centering
    \includegraphics[width=0.27\textwidth]{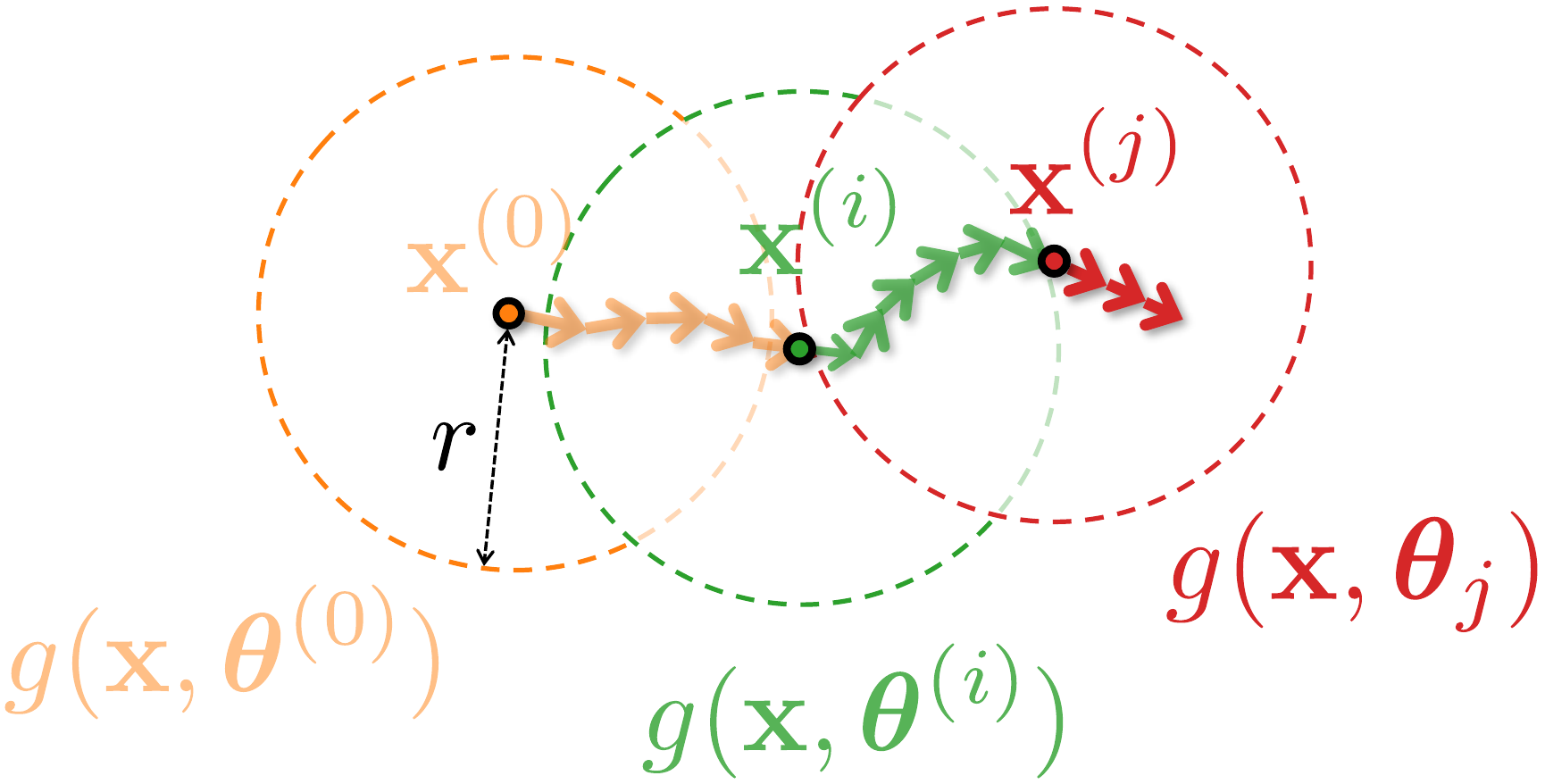}
    \caption{Plot of a hypothetical optimization path (as arrows) of LNSO initialized at $\mbf{x}^{(0)}$. Input $\mbf{x}$ is $2$-dimensional. The boundary of a ball $\mcl{B}$ of radius $r$ is plotted as a dotted circle. When the updated $\mbf{x}$ moves out of $\mcl{B}$, the center of $\mcl{B}$ and $\bm{\theta}$ are updated.}
    \label{fig:lnso}
\end{figure}

{\bf{Search of Lacing Values.}} 
Given a continuous random variable $\mbf{Z}$, to find an LV w.r.t $\mbf{x}_t$ in line 4 of Algorithm~\ref{alg:v-ucb}, i.e., to find a $\mbf{z}$ satisfying $d_u(\mbf{z}) \triangleq u_{t-1}(\mbf{x}_t, \mbf{z}) - V_\alpha(u_{t-1}(\mbf{x}_t,\mbf{Z})) \ge 0$ and $d_l(\mbf{z}) \triangleq V_\alpha(l_t(\mbf{x}_t,\mbf{Z})) - l_{t-1}(\mbf{x}_t, \mbf{z}) \ge 0$, we choose a $\mbf{z}$ that minimizes
%
\begin{equation}
\mcl{L}_{\text{LV}}(\mbf{z}) \triangleq \text{ReLU}(-d_u(\mbf{z})) + \text{ReLU}(-d_l(\mbf{z}))
\label{eq:llv}
\end{equation}
where $\text{ReLU}(\omega) = \max(\omega,0)$ is the rectified linear unit function ($\omega \in \mbb{R}$). To include the heuristic in Section~\ref{subsec:improvedselectz} which selects the LV with the highest probability density, we find $\mbf{z}$ that minimizes
\begin{equation*}
\mcl{L}_{\text{LV-P}}(\mbf{z}) \triangleq \mcl{L}_{\text{LV}}(\mbf{z}) - \mbb{I}_{d_u(\mbf{z}) \ge 0} \mbb{I}_{ d_l(\mbf{z}) \ge 0}\ p(\mbf{z})
\end{equation*}
where $\mcl{L}_{\text{LV}}(\mbf{z})$ is defined in \eqref{eq:llv}; $p(\mbf{z})$ is the probability density; $\mbb{I}_{d_u(\mbf{z}) \ge 0}$ and $\mbb{I}_{ d_l(\mbf{z}) \ge 0}$ are indicator functions.

\section{Experiments}
\label{sec:experiments}

In this section, we empirically evaluate the performance of {\vy}. The work of \citet{borisk20} has motivated the use of the approximated variant of their algorithm {\kg} over its original version {\kgfull} by showing that {\kg} achieves comparable empirical performances to {\kgfull} while incurring much less computational cost. Furthermore, {\kg} has been shown to significantly outperform 
other competing algorithms \cite{borisk20}.
Therefore, we choose {\kg} as the main baseline to empirically compare with {\vy}. The experiments using {\kg} is performed by 
adding
new objective functions to the existing implementation of \citet{borisk20} at \texttt{https://github.com/saitcakmak/BoRisk}.

Regarding {\vy}, when $\dz$ is finite and the distribution of $\mbf{Z}$ is not uniform, we perform {\vy} by selecting $\mbf{z}_t$ as an LV at random, labeled as \emph{{\vy} Unif}, and by selecting $\mbf{z}_t$ as the LV with the maximum probability mass, labeled as \emph{{\vy} Prob}.

The performance metric 
is defined as $V_\alpha(f(\mbf{x}_*,\mbf{Z})) - V_\alpha(f(\tilde{\mbf{x}},\mbf{Z}))$ where $\tilde{\mbf{x}}$ is the recommended input. 
The evaluation of {\var} is described in Section~\ref{sec:continuousz}.
The recommended input is $\argmax_{\mbf{x} \in \mcl{D}_T} V_{\alpha}(\mu_{t-1}(\mbf{x}, \mbf{Z}))$ for {\vy}, and $\argmin_{\mbf{x} \in \dx} \mbb{E}_{t-1}[V_{\alpha}(f(\mbf{x}, \mbf{Z}))]$ for {\kg} \cite{borisk20}, where $\mbb{E}_{t-1}$ is the conditional expectation over the unknown $f$ given the observations $\mcl{D}_{t-1}$ (approximated by a finite set of functions sampled from the GP posterior belief).\footnote{While the work of \citet{borisk20} considers a minimization problem of {\var}, our work considers a maximization problem of {\var}. Therefore, the objective functions for {\kg} are the negation of those for {\vy}. For {\vy} at risk level $\alpha$, 
the risk level for {\kg} is $1-\alpha$.}
We repeat each experiment $10$ times with different random initial observations $\mbf{y}_{\mcl{D}_0}$ and plot both the mean (as lines) and the $70\%$ confidence interval (as shaded areas) of the $\log10$ of the performance metric.
The detailed descriptions of experiments are deferred to Appendix~\ref{app:experiment}.

\subsection{Synthetic Benchmark Functions}
\label{subsec:syn}

We use $3$ synthetic benchmark functions: Branin-Hoo, Goldstein-Price, and Hartmann-3D functions to construct $4$ optimization problems: Branin-Hoo-$(1,1)$, Goldstein-Price-$(1,1)$, Hartmann-$(1,2)$, and Hartmann-$(2,1)$.
The tuples represent $(d_x, d_z)$ corresponding to the 
dimensions of $\mbf{x}$ and $\mbf{z}$. 
The noise variance $\sigma_n^2$ is set to $0.01$. The risk level $\alpha$ is $0.1$. 
There are $2$ different settings: finite $\dz$ ($|\dz| = 64$ for Hartmann-$(1,2)$ and $|\dz|=100$ for the others) and continuous $\dz$. In the latter setting, $r=0.1$ and the surrogate function is a neural network of $2$ hidden layers with $30$ hidden neurons each, and sigmoid activation functions.

The results are shown in Fig.~\ref{fig:synfinite} and Fig.~\ref{fig:syncont} for the settings of discrete $\dz$ and continuous $\dz$, respectively. When $\dz$ is discrete (Fig.~\ref{fig:synfinite}), {\vy} Unif is on par with {\kg} in optimizing Branin-Hoo-$(1,1)$ and Goldstein-Price-$(1,1)$, and outperforms {\kg} in optimizing Hartmann-$(1,2)$ and Hartmann-$(2,1)$. {\vy} Prob is also able to exploit the probability distribution of $\mbf{Z}$ to outperform {\vy} Unif. When $\dz$ is continuous (Fig.~\ref{fig:syncont}), {\vy} Prob outperforms {\kg}.
The unsatisfactory performance of {\kg} in some experiments may be attributed to its approximation of the inner optimization problem in the acquisition function \cite{borisk20}, and the approximation of {\var} using samples of $\mbf{Z}$ and the GP posterior belief.
\newcommand\figheight{0.201}
\begin{figure}[ht]
\centering
\begin{tabular}{@{}cc@{}}
\includegraphics[height=\figheight\textwidth]{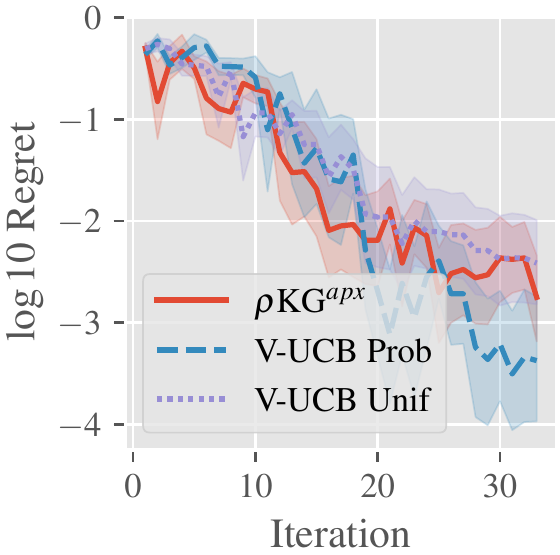}
&
\includegraphics[height=\figheight\textwidth]{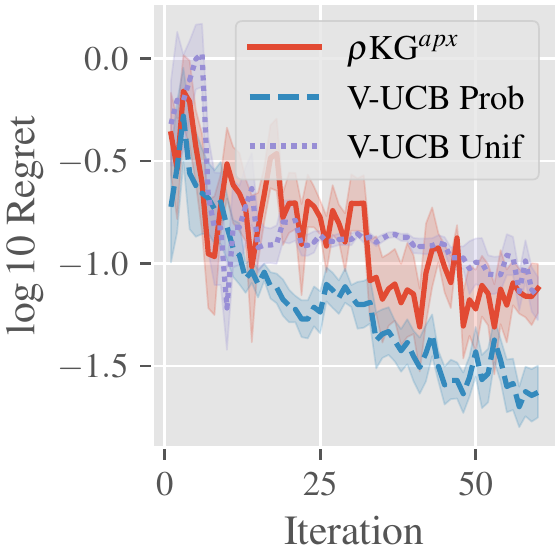}
\\
(a) Branin-Hoo-$(1,1)$
&
(b) Goldstein-Price-$(1,1)$
\\
\\
\includegraphics[height=\figheight\textwidth]{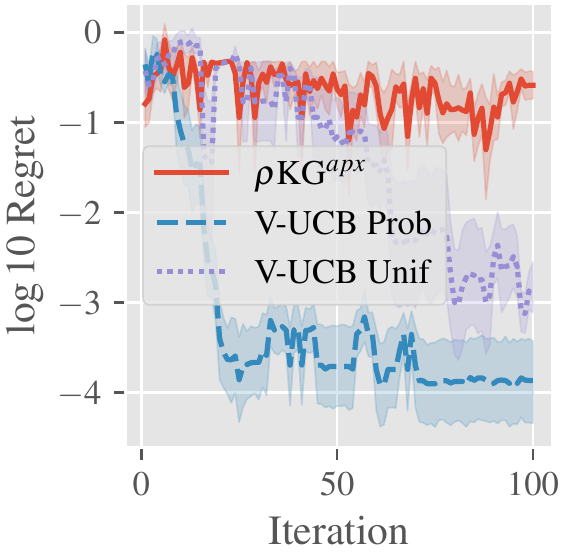}
&
\includegraphics[height=\figheight\textwidth]{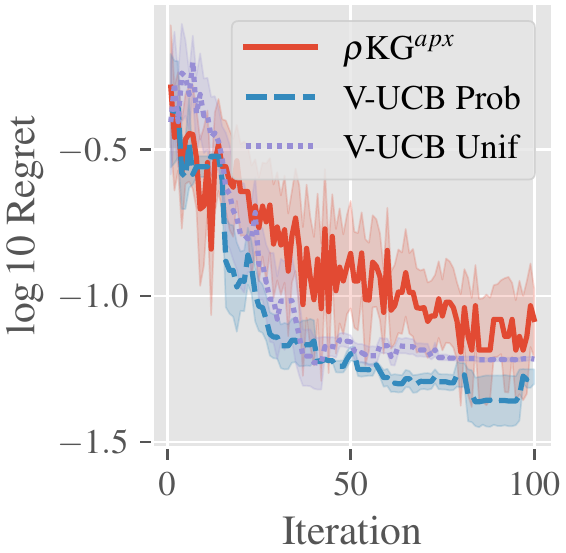}
\\
(c) Hartmann-$(1,2)$
&
(d) Hartmann-$(2,1)$
\end{tabular}
\caption{Synthetic benchmark functions with finite $\dz$.}
\label{fig:synfinite}
\end{figure}
\begin{figure}[ht]
\centering
\begin{tabular}{@{}cc@{}}
\includegraphics[height=\figheight\textwidth]{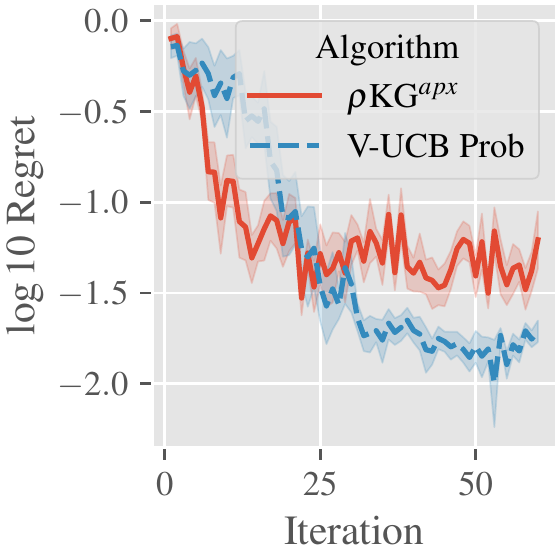}
&
\includegraphics[height=\figheight\textwidth]{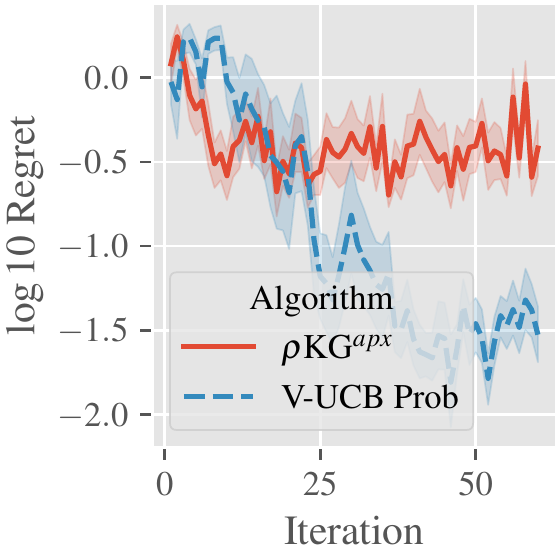}
\\
(a) Branin-Hoo-$(1,1)$
&
(b) Goldstein-Price-$(1,1)$
\\
\\
\includegraphics[height=\figheight\textwidth]{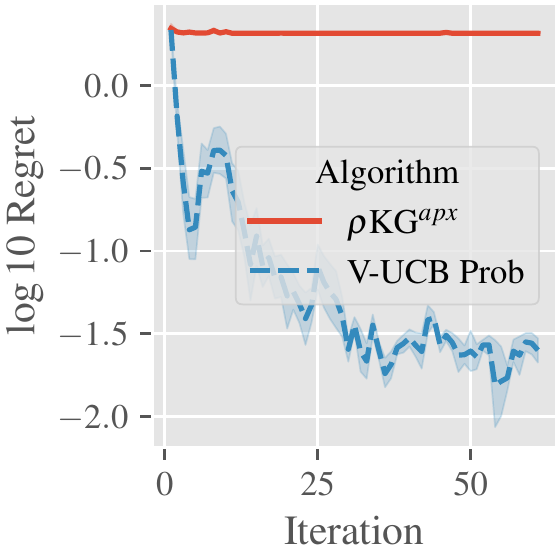}
&
\includegraphics[height=\figheight\textwidth]{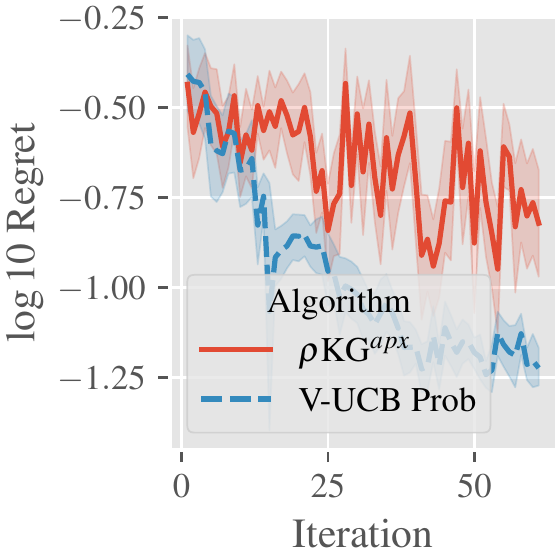}
\\
(c) Hartmann-$(1,2)$
&
(d) Hartmann-$(2,1)$
\end{tabular}
\caption{Synthetic benchmark functions with continuous $\dz$.}
\label{fig:syncont}
\end{figure}
\subsection{Simulated Optimization Problems}
\label{subsec:simulated}


The first problem is portfolio optimization adopted by \cite{borisk20}. 
There are $d_x = 3$ optimization variables (risk and trade aversion parameters, and the holding cost multiplier) and $d_z = 2$ environmental random variables (bid-ask spread and borrow cost). The variable $\mbf{Z}$ follows a discrete uniform distribution with $|\dz| = 100$. Hence, there is no difference between {\vy} Unif and {\vy} Prob. Thus, we only report the results of the latter.
The objective function is the posterior mean of a trained GP on the dataset in \citet{borisk20} of size $3000$ generated from CVXPortfolio. The noise variance $\sigma_n^2$ is set to $0.01$. The risk level $\alpha$ is set to $0.2$.

The second problem is a simulated robot pushing task for which we use the implementation from the work of \citet{wang17mes}.
The simulation is viewed as a $3$-dimensional function $\mbf{h}(r_x, r_y, t_p)$ returning the 2-D location of the pushed object, where $r_x,r_y \in [-5,5]$ are the robot location and $t_p \in [1,30]$ is the pushing duration. 
The objective is to minimize the distance to a fixed goal location $\mbf{g}=(g_x,g_y)^\top$, i.e., the objective function of the maximization problem is $f_0(r_x, r_y, t_p) = -\Vert \mbf{h}(r_x, r_y, t_p) - \mbf{g}\Vert$. We assume that there are perturbations in specifying the robot location $W_x, W_y$ whose support $\dz$ includes $64$ equi-distant points in $[-1,1]^2$ and whose probability mass is proportional to $\exp(-(W_x^2 + W_y^2) / 0.4^2)$. It leads to a random objective function $f(r_x, r_y, t_p, W_x, W_y) \triangleq f_0(r_x + W_x, r_y + W_y, t_p)$. 
We aim to maximize the {\var} of $f$ which is more difficult than maximizing that of $f_0$. 
Moreover, a random noise following $\mcl{N}(0,0.01)$ is added to the evaluation of $f$. 
The risk level $\alpha$ is set to $0.1$.

The results are shown in Fig.~\ref{fig:portrobot}. 
We observe that {\vy} outperforms {\kg} in both problems. Furthermore, in comparison to our synthetic experiments, the difference between {\vy} Unif and {\vy} Prob is not significant in the robot pushing experiment. This is because the chance that a uniform sample of LV has a large probability mass is higher in the robot pushing experiment due to a larger region of $\dz$ having high probabilities.

\begin{figure}[ht]
\centering
\begin{tabular}{@{}cc@{}}
\includegraphics[height=\figheight\textwidth]{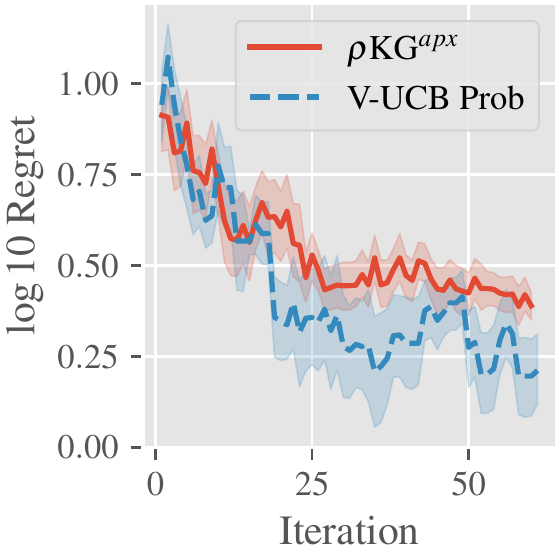}
&
\includegraphics[height=\figheight\textwidth]{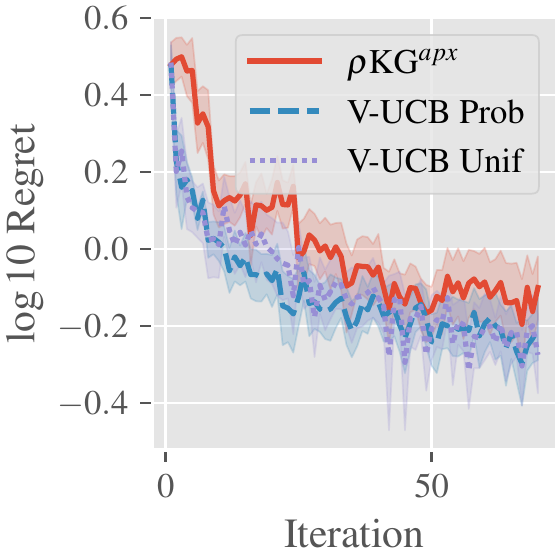}
\\
(a) Portfolio optimization $(3,2)$
&
(b) Robot pushing $(3,2)$
\end{tabular}
\caption{Simulated experiments.}
\label{fig:portrobot}
\end{figure}

\section{Conclusion}
\label{sec:conclusion}

To tackle the BO of {\var} problem, we construct a no-regret algorithm, namely \emph{{\var} upper confidence bound} (V-UCB), through the design of a confidence bound of {\var} and a set of \emph{lacing values} (LV) that is guaranteed to exist. Besides, we introduce a heuristic to select an LV that improves the emprical performance of {\vy} over random selection of LV. We also draw 
an elegant
connection between BO of {\var} and adversarially robust BO in terms of both problem formulation and solutions. Lastly, we provide practical techniques for implementing {\var} with continuous $\mbf{Z}$. 
While {\vy} is more computationally efficient than the the state-of-the-art {\kg} algorithm for BO of {\var}, it also demonstrates competitive empirical performances in our experiments.

\bibliography{bo}
\bibliographystyle{icml2021}

\clearpage 
\appendix

\section{Proof of Lemma~\ref{lemma:confbound}}
\label{app:proofvconfbound}

{\bf{Lemma~\ref{lemma:confbound}.}}
Similar to the definition of $f(\mbf{x},\mbf{Z})$, let $l_{t-1}(\mbf{x}, \mbf{Z})$ and $u_{t-1}(\mbf{x},\mbf{Z})$ denote the random function over $\mbf{x}$ where the randomness comes from the random variable $\mbf{Z}$; $l_{t-1}$ and $u_{t-1}$ are defined in \eqref{eq:fbound}.
Then, $\forall \mbf{x} \in \mcl{D}_{\mbf{x}}$, $t \ge 1$, $\alpha \in (0,1)$,
\[
\begin{array}{r@{}l}
\displaystyle V_\alpha(f(\mbf{x},\mbf{Z})) \displaystyle &\displaystyle \in I_{t-1}[V_{\alpha}(f(\mbf{x},\mbf{Z}))]\\
	 &\displaystyle \triangleq [V_\alpha(l_{t-1}(\mbf{x},\mbf{Z})), V_\alpha(u_{t-1}(\mbf{x},\mbf{Z}))]
\end{array}
\]
holds with probability $\ge 1 - \delta$ for $\delta$ in Lemma~\ref{lemma:ucb51}, where $V_\alpha(l_{t-1}(\mbf{x},\mbf{Z}))$ and $V_\alpha(u_{t-1}(\mbf{x},\mbf{Z}))$ are defined as \eqref{eq:var}.

\begin{proof}
Conditioned on the event $f(\mbf{x},\mbf{z}) \in I_{t-1}[f(\mbf{x},\mbf{z})] \triangleq [l_{t-1}(\mbf{x},\mbf{z}), u_{t-1}(\mbf{x},\mbf{z})]$ for all $\mbf{x} \in \dx$, $\mbf{z} \in \dz$, $t \ge 1$ which occurs with probability $\ge 1 - \delta$ for $\delta$ in Lemma~\ref{lemma:ucb51}, we will prove that $V_\alpha(l_{t-1}(\mbf{x},\mbf{Z})) \le V_\alpha(f(\mbf{x},\mbf{Z}))$. The proof of $V_\alpha(f(\mbf{x},\mbf{Z})) \le V_\alpha(u_{t-1}(\mbf{x},\mbf{Z}))$ can be done in a similar manner.

From $f(\mbf{x},\mbf{z}) \in I_{t-1}[f(\mbf{x},\mbf{z})] \triangleq [l_{t-1}(\mbf{x},\mbf{z}), u_{t-1}(\mbf{x},\mbf{z})]$ for all $\mbf{x} \in \dx$, $\mbf{z} \in \dz$, $t \ge 1$ we have $\forall \mbf{x} \in \dx, \mbf{z} \in \dz, t \ge 1,$
\begin{align*}
f(\mbf{x},\mbf{z}) &\ge l_{t-1}(\mbf{x},\mbf{z})\ .
\end{align*}
Therefore, for all $\omega \in \mbb{R}$, $\mbf{x} \in \dx$, $\mbf{z} \in \dz$, $t \ge 1$,
\begin{align*}
f(\mbf{x},\mbf{z}) \le \omega &\Rightarrow l_{t-1}(\mbf{x},\mbf{z}) \le \omega\\
P(f(\mbf{x},\mbf{Z}) \le \omega)  &\le P(l_{t-1}(\mbf{x},\mbf{Z}) \le \omega)\ .
\end{align*}
So, for all $\omega \in \mbb{R}$, $\alpha \in (0,1)$, $\mbf{x} \in \dx$, $t \ge 1$,
\begin{align*}
P(f(\mbf{x},\mbf{Z}) \le \omega) \ge \alpha \Rightarrow P(l_{t-1}(\mbf{x},\mbf{Z}) \le \omega) \ge \alpha\ .
\end{align*}
Hence, the set $\{\omega: P(f(\mbf{x},\mbf{Z}) \le \omega) \ge \alpha\}$ is a subset of $\{\omega: P(l_{t-1}(\mbf{x},\mbf{Z}) \le \omega) \ge \alpha\}$ for all $\alpha \in (0,1)$, $\mbf{x} \in \dx$, $t \ge 1$, which implies that $\inf\{\omega: P(f(\mbf{x},\mbf{Z}) \le \omega) \ge \alpha\} \ge \inf\{\omega: P(l_{t-1}(\mbf{x},\mbf{Z}) \le \omega) \ge \alpha\}$, i.e., $V_\alpha(l_{t-1}(\mbf{x},\mbf{Z})) \le V_\alpha(f(\mbf{x},\mbf{Z}))$ for all $\alpha \in (0,1)$, $\mbf{x} \in \dx$, $t \ge 1$.
\end{proof}

\section{Proof of \eqref{eq:iregretbound1}}
\label{app:iregretbound1}

We prove that 
\begin{align*}
r(\mbf{x}_t) \le V_\alpha(u_{t-1}(\mbf{x}_t,\mbf{Z})) - V_\alpha(l_{t-1}(\mbf{x}_t,\mbf{Z}))\ \forall t \ge 1
\end{align*}
which holds with probability $\ge 1 - \delta$ for $\delta$ in Lemma~\ref{lemma:ucb51}.

\begin{proof}
Conditioned on the event $V_\alpha(f(\mbf{x},\mbf{Z})) \in I_{t-1}[V_{\alpha}(f(\mbf{x},\mbf{Z}))] \triangleq [V_\alpha(l_{t-1}(\mbf{x},\mbf{Z})), V_\alpha(u_{t-1}(\mbf{x},\mbf{Z}))]$ for all $\alpha \in (0,1)$, $\mbf{x} \in \dx$, $t \ge 1$, which occurs with probability $\ge 1 - \delta$ in Lemma~\ref{lemma:confbound},
\begin{align*}
V_\alpha(f(\mbf{x}_*,\mbf{Z})) &\le V_\alpha(u_{t-1}(\mbf{x}_*,\mbf{Z}))\\
V_\alpha(f(\mbf{x}_t,\mbf{Z})) &\ge V_\alpha(l_{t-1}(\mbf{x}_t,\mbf{Z}))\ .
\end{align*}
Hence,
\begin{align}
r(\mbf{x}_t) &\triangleq V_\alpha(f(\mbf{x}_*,\mbf{Z})) - V_\alpha(f(\mbf{x}_t,\mbf{Z}))\nonumber\\
	&\le V_\alpha(u_{t-1}(\mbf{x}_*,\mbf{Z})) - V_\alpha(l_{t-1}(\mbf{x}_t,\mbf{Z}))\ .\label{eq:inter1irb1}
\end{align}
Since $\mbf{x}_t$ is selected as $\argmax_{\mbf{x} \in \dx} V_\alpha(u_{t-1}(\mbf{x},\mbf{Z}))$, 
\begin{align*}
V_\alpha(u_{t-1}(\mbf{x}_*,\mbf{Z})) \le V_\alpha(u_{t-1}(\mbf{x}_t,\mbf{Z}))\ ,
\end{align*}
equivalently, $V_\alpha(u_{t-1}(\mbf{x}_*,\mbf{Z})) - V_\alpha(l_{t-1}(\mbf{x}_t,\mbf{Z}))
\le V_\alpha(u_{t-1}(\mbf{x}_t,\mbf{Z})) - V_\alpha(l_{t-1}(\mbf{x}_t,\mbf{Z}))$. Hence, from \eqref{eq:inter1irb1}, $r(\mbf{x}_t) \le  V_\alpha(u_{t-1}(\mbf{x}_t,\mbf{Z})) - V_\alpha(l_{t-1}(\mbf{x}_t,\mbf{Z}))$ for all $\alpha \in (0,1)$ and $t \ge 1$.
\end{proof}

\section{Proof of Theorem~\ref{theorem:lv}}
\label{app:prooflv}

{\bf{Theorem~\ref{theorem:lv}}.}
$\forall \alpha \in (0,1)$, $\forall \mbf{x} \in \dx$, $\forall t \ge 1$, there exists a lacing value in $\dz$ with respect to $\mbf{x}$ and $t$.

\begin{proof}
Recall that 
\begin{align*}
\mcl{Z}_l^{\le}&\triangleq \{\mbf{z} \in \mcl{D}_{\mbf{z}}: l_{t-1}(\mbf{x},\mbf{z}) \le V_{\alpha}(l_{t-1}(\mbf{x}, \mbf{Z}))\}\ .
\end{align*}

From to the definition of $\mcl{Z}_l^\le$ and $V_{\alpha}(l_{t-1}(\mbf{x},\mbf{Z}))$, we have
\begin{align}
P(\mbf{Z} \in \mcl{Z}_l^\le) \ge \alpha\ .
\label{eq:pzl}
\end{align}
Since $\alpha \in (0,1)$, $\mcl{Z}_l^\le \neq \emptyset$.
We prove the existence of LV by contradiction: (a) assuming that $\exists \mbf{x} \in \dx, \exists t \ge 1, \forall \mbf{z} \in \mcl{Z}_l^{\le}, u_{t-1}(\mbf{x}, \mbf{z}) < V_{\alpha}(u_{t-1}(\mbf{x}, \mbf{Z}))$ and then, (b) proving that  $V_\alpha(u_{t-1}(\mbf{x},\mbf{Z}))$ is not a lower bound of $\{\omega: P(u_{t-1}(\mbf{x},\mbf{Z}) \le \omega) \ge \alpha\}$ which is a contradiction.



Since the GP posterior mean $\mu_{t-1}$ and posterior standard deviation $\sigma_{t-1}$ are continuous functions in $\dx \times \dz$,  $l_{t-1}$ and $u_{t-1}$ are continuous functions in the closed $\dz \subset \mbb{R}^{d_z}$ ($\mbf{x}$ and $t$ are given and remain fixed in this proof). 
We will prove that $\mcl{Z}_l^\le$ is closed in $\mbb{R}^{d_z}$ by contradiction.

If $\mcl{Z}_l^\le$ is not closed in $\mbb{R}^{d_z}$, there exists a limit point $\mbf{z}_p$ of $\mcl{Z}_l^\le$ such that $\mbf{z}_p \notin \mcl{Z}_l^\le$. Since $\mcl{Z}_l^\le \subset \dz$ and $\dz$ is closed in $\mbb{R}^{d_z}$, $\mbf{z}_p \in \dz$. Thus, for $\mbf{z}_p \notin \mcl{Z}_l^\le$, $l_{t-1}(\mbf{x},\mbf{z}_p) > V_\alpha(l_{t-1}(\mbf{x},\mbf{Z}))$ (from the definition of $\mcl{Z}_l^\le$). Then, there exists $\epsilon_0 > 0$ such that $l_{t-1}(\mbf{x},\mbf{z}_p) > V_\alpha(l_{t-1}(\mbf{x},\mbf{Z})) + \epsilon_0$. The pre-image of the open interval $I_0 = (l_{t-1}(\mbf{x},\mbf{z}_p) - \epsilon_0/2, l_{t-1}(\mbf{x},\mbf{z}_p) + \epsilon_0/2)$ under $l_{t-1}$ is also an open set $\mcl{V}$ containing $\mbf{z}_p$ (because $l_{t-1}$ is a continuous function). Since $\mbf{z}_p$ is a limit point of $\mcl{Z}_l^\le$, there exists an $\mbf{z}' \in \mcl{Z}_l^\le \cap \mcl{V}$. Then, $l_{t-1}(\mbf{x},\mbf{z}') \in I_0$, so $l_{t-1}(\mbf{x},\mbf{z}') \ge l_{t-1}(\mbf{x},\mbf{z}_p) - \epsilon_0/2 \ge V_\alpha(l_{t-1}(\mbf{x},\mbf{Z})) + \epsilon_0 - \epsilon_0 /2 = V_\alpha(l_{t-1}(\mbf{x},\mbf{Z})) + \epsilon_0/2$. It contradicts $\mbf{z}' \in \mcl{Z}_l^\le$.

Therefore, $\mcl{Z}_l^\le$ is a closed set in $\mbb{R}^{d_z}$. Besides, since $\{u_{t-1}(\mbf{x},\mbf{z}): \mbf{z} \in \mcl{Z}_l^\le\}$ is upper bounded by $V_\alpha(u_{t-1}(\mbf{x},\mbf{Z}))$ (due to our assumption), so $\sup\{ u_{t-1}(\mbf{x},\mbf{z}): \mbf{z} \in \mcl{Z}_l^\le \}$ exists.
Let  $\mbf{z}_l^+$ be such that $u_{t-1}(\mbf{x},\mbf{z}_l^+) \triangleq \sup\{ u_{t-1}(\mbf{x},\mbf{z}): \mbf{z} \in \mcl{Z}_l^\le \}$. Then, $\mbf{z}_l^+ \in \mcl{Z}_l^\le$ because $\mcl{Z}_l^\le$ is closed.

Moreover, from our assumption that $\forall \mbf{z} \in \mcl{Z}_l^{\le}, u_{t-1}(\mbf{x}, \mbf{z}) < V_{\alpha}(u_{t-1}(\mbf{x}, \mbf{Z}))$, we have $u_{t-1}(\mbf{x},\mbf{z}_l^+) < V_{\alpha}(u_{t-1}(\mbf{x}, \mbf{Z}))$. Furthermore, 
\begin{align*}
    P(u_{t-1}(\mbf{x}, \mbf{Z}) \le u_{t-1}(\mbf{x}, \mbf{z}_l^+)) \ge P(\mbf{Z} \in \mcl{Z}_l^{\le}) \ge \alpha\ .
\end{align*}
 where the first inequality is because $u_{t-1}(\mbf{x},\mbf{z}_l^+) = \sup\{ u_{t-1}(\mbf{x},\mbf{z}): \mbf{z} \in \mcl{Z}_l^\le \}$ and the last inequality is from \eqref{eq:pzl}.  Hence, $V_\alpha(u_{t-1}(\mbf{x},\mbf{Z})$ is not a lower bound of $\{\omega: P(u_{t-1}(\mbf{x},\mbf{Z}) \le \omega) \ge \alpha\}$.
\end{proof}

\section{Proof of Lemma~\ref{lemma:iregretbound2}}
\label{app:iregretbound2}

{\bf{Lemma~\ref{lemma:iregretbound2}.}}
By selecting $\mbf{x}_t$ as a maximizer of $V_\alpha(u_{t-1}(\mbf{x},\mbf{Z}))$ and selecting $\mbf{z}_t$ as an LV w.r.t $\mbf{x}_t$, the instantaneous regret is upper-bounded by:
\begin{equation*}
r(\mbf{x}_t) \le 2 \beta_t^{1/2} \sigma_{t-1}(\mbf{x}_t, \mbf{z}_t)\ \forall t \ge 1
\end{equation*}
with probability $\ge 1 - \delta$ for $\delta$ in Lemma~\ref{lemma:ucb51}.

\begin{proof}
Conditioned on the event $f(\mbf{x},\mbf{z}) \in I_{t-1}[f(\mbf{x},\mbf{z})] \triangleq [l_{t-1}(\mbf{x},\mbf{z}), u_{t-1}(\mbf{x},\mbf{z})]$ for all $\mbf{x} \in \dx$, $\mbf{z} \in \dz$, $t \ge 1$ which occurs with probability $\ge 1 - \delta$ in Lemma~\ref{lemma:ucb51}, it follows that $V_\alpha(f(\mbf{x},\mbf{Z})) \in I_{t-1}[V_{\alpha}(f(\mbf{x},\mbf{Z}))] \triangleq [V_\alpha(l_{t-1}(\mbf{x},\mbf{Z})), V_\alpha(u_{t-1}(\mbf{x},\mbf{Z}))]$ for all $\alpha \in (0,1)$, $\mbf{x} \in \dx$, and $t \ge 1$ in Lemma~\ref{lemma:confbound}.

From \eqref{eq:iregretbound1}, by selecting $\mbf{z}_t$ as an LV, for all $t \ge 1$,

\begin{align*}
r(\mbf{x}_t) &\le V_\alpha(u_{t-1}(\mbf{x}_t,\mbf{Z})) - V_\alpha(l_{t-1}(\mbf{x}_t,\mbf{Z}))\\
	&\le u_{t-1}(\mbf{x}_t, \mbf{z}_t) - l_{t-1}(\mbf{x}_t, \mbf{z}_t) \text{  (since } \mbf{z}_t \text{ is an LV)}\\
	&\le \mu_{t-1}(\mbf{x}_t,\mbf{z}_t) + \beta_t^{1/2} \sigma_{t-1}(\mbf{x}_t,\mbf{z}_t) \\
	&\quad- \mu_{t-1}(\mbf{x}_t,\mbf{z}_t) + \beta_t^{1/2} \sigma_{t-1}(\mbf{x}_t,\mbf{z}_t)\\
	&= 2 \beta_t^{1/2} \sigma_{t-1}(\mbf{x}_t,\mbf{z}_t)\ .
\end{align*}
\end{proof}

\section{Bound on the Average Cumulative Regret}
\label{app:rtbound}

Conditioned on the event $f(\mbf{x},\mbf{z}) \in I_{t-1}[f(\mbf{x},\mbf{z})] \triangleq [l_{t-1}(\mbf{x},\mbf{z}), u_{t-1}(\mbf{x},\mbf{z})]$ for all $\mbf{x} \in \dx$, $\mbf{z} \in \dz$, $t \ge 1$ which occurs with probability $\ge 1 - \delta$ in Lemma~\ref{lemma:ucb51}, it follows that $r(\mbf{x}_t) \le 2 \beta_t^{1/2} \sigma_{t-1}(\mbf{x}_t, \mbf{z}_t)\ \forall t \ge 1$ in Lemma~\ref{lemma:iregretbound2}. Therefore, 
\begin{align*}
R_T &\triangleq \sum_{t=1}^T r(\mbf{x}_t)
	\le \sum_{t=1}^T 2 \beta_t^{1/2} \sigma_{t-1}(\mbf{x}_t,\mbf{z}_t)\\
	&\le 2 \beta_T^{1/2} \sum_{t=1}^T \sigma_{t-1}(\mbf{x}_t,\mbf{z}_t)
\end{align*}
since $\beta_t$ is a non-decreasing sequence.

From Lemma 5.4 and the Cauchy-Schwarz inequality in \cite{srinivas10ucb}, we have 
\begin{align}
2 \sum_{t=1}^T \sigma_{t-1}(\mbf{x}_t,\mbf{z}_t) \le \sqrt{C_1 T \gamma_T}
\label{eq:frombogu}
\end{align}
where $C_1 = 8 / \log(1 + \sigma_n^{-2})$.
Hence,
\begin{align*}
R_T \le \sqrt{C_1 T \beta_T \gamma_T}\ .
\end{align*}
Equivalently, $R_T / T \le \sqrt{C_1 \beta_T \gamma_T/ T} \ .$
Since $\gamma_T$ is shown to be bounded for several common kernels in \cite{srinivas10ucb}, the above implies that $\text{lim}_{T \rightarrow \infty} R_T / T = 0$, i.e., the algorithm is no-regret.

\section{Bound on $r(\mbf{x}_{t_*(T)})$}
\label{app:recommendxbound}
Conditioned on the event $f(\mbf{x},\mbf{z}) \in I_{t-1}[f(\mbf{x},\mbf{z})] \triangleq [l_{t-1}(\mbf{x},\mbf{z}), u_{t-1}(\mbf{x},\mbf{z})]$ for all $\mbf{x} \in \dx$, $\mbf{z} \in \dz$, $t \ge 1$, which occurs with probability $\ge 1 - \delta$ in Lemma~\ref{lemma:ucb51}, it follows that $V_\alpha(f(\mbf{x},\mbf{Z})) \in I_{t-1}[V_{\alpha}(f(\mbf{x},\mbf{Z}))] \triangleq [V_\alpha(l_{t-1}(\mbf{x},\mbf{Z})), V_\alpha(u_{t-1}(\mbf{x},\mbf{Z}))]$ for all $\alpha \in (0,1), \mbf{x} \in \dx, t \ge 1$ in Lemma~\ref{lemma:confbound}. Furthermore, we select $\mbf{z}_t$ as an LV, so $l_{t-1}(\mbf{x}_t, \mbf{z}_t) \le V_\alpha(l_{t-1}(\mbf{x}_t,\mbf{Z})) \le V_\alpha(u_{t-1}(\mbf{x}_t,\mbf{Z})) \le u_{t-1}(\mbf{x}_t, \mbf{z}_t)$ according to the Definition~\ref{definition:lv}. 

At $T$-th iteration, by recommending $\mbf{x}_{t_*(T)}$ as an estimate of $\mbf{x}_*$ where $t_*(T) \triangleq \argmax_{t \in \{1,\dots,T\}} V_{\alpha}(l_{t-1}(\mbf{x}_t,\mbf{Z}))$, we have
\begin{align*}
V_{\alpha}(l_{t_*(T)-1}(\mbf{x}_{t_*(T)}, \mbf{Z})) 
&= \max_{t \in \{1,\dots,T\}} V_{\alpha}(l_{t-1}(\mbf{x}_t, \mbf{Z}))\\
&\ge \frac{1}{T} \sum_{t=1}^T V_{\alpha}(l_{t-1}(\mbf{x}_t, \mbf{Z}))\ .
\end{align*}
Therefore,
\begin{align*}
r(\mbf{x}_{t_*(T)}) &= V_\alpha(f(\mbf{x}_*,\mbf{Z})) - V_{\alpha}(f(\mbf{x}_{t_*(T)}, \mbf{Z}))\\
	&\le V_\alpha(f(\mbf{x}_*,\mbf{Z})) - V_{\alpha}(l_{t_*(T)-1}(\mbf{x}_{t_*(T)}, \mbf{Z}))\\
	&\le \frac{1}{T} \sum_{t=1}^T \left(V_\alpha(f(\mbf{x}_*,\mbf{Z})) - V_{\alpha}(l_{t-1}(\mbf{x}_t, \mbf{Z})) \right)\ .
\end{align*}
Furthermore, $V_\alpha(f(\mbf{x}_*,\mbf{Z})) \le V_\alpha(u_{t-1}(\mbf{x}_*,\mbf{Z}))$ from our condition, so
\begin{align*}
&r(\mbf{x}_{t_*(T)}) \le \frac{1}{T} \sum_{t=1}^T \left(V_\alpha(f(\mbf{x}_*,\mbf{Z})) - V_{\alpha}(l_{t-1}(\mbf{x}_t, \mbf{Z})) \right)\\
	&\le \frac{1}{T} \sum_{t=1}^T \left(V_\alpha(u_{t-1}(\mbf{x}_*,\mbf{Z})) - V_{\alpha}(l_{t-1}(\mbf{x}_t, \mbf{Z})) \right)\\
	&\le \frac{1}{T} \sum_{t=1}^T \left(V_\alpha(u_{t-1}(\mbf{x}_t,\mbf{Z})) - V_{\alpha}(l_{t-1}(\mbf{x}_t, \mbf{Z})) \right)\\
	&\le \frac{1}{T} \sum_{t=1}^T \left(u_{t-1}(\mbf{x}_t,\mbf{z}_t) - l_{t-1}(\mbf{x}_t, \mbf{z}_t) \right) \text{ (since } \mbf{z}_t \text{ is an LV)}\\
	&\le \frac{1}{T} \sum_{t=1}^T 2 \beta_t^{1/2} \sigma_{t-1}(\mbf{x}_t,\mbf{z}_t)\\
	&\le \sqrt{\frac{C_1 \beta_T \gamma_T}{T}} \text{ (from Appendix~\ref{app:rtbound})}\ .
\end{align*}

Since $\gamma_T$ is shown to be bounded for several common kernels in \cite{srinivas10ucb}, the above implies that $\text{lim}_{T \rightarrow \infty} r(\mbf{x}_{t_*(T)}) = 0$.

\section{Proof of Theorem~\ref{theorem:0plus} and Its Corollaries}
\label{app:0plus}

\subsection{Proof of Theorem~\ref{theorem:0plus}}

{\bf{Theorem~\ref{theorem:0plus}.}}
Let $\mbf{W}$ be a random variable with the support $\mcl{D}_w \subset \mbb{R}^{d_w}$ and dimension $d_w$.
Let $h$ be a continuous function mapping from $\mbf{w} \in \mcl{D}_w$ to $\mbb{R}$. Then, $h(\mbf{W})$ denotes the random variable whose realization is the function $h$ evaluation at a realization $\mbf{w}$ of $\mbf{W}$. Suppose $h(\mbf{w})$ has a minimizer $\mbf{w}_{\min} \in \mcl{D}_w$, then $\lim_{\alpha \rightarrow 0^+} V_{\alpha}(h(\mbf{W})) = h(\mbf{w}_{\min})\ .$

Recall that the support $\mcl{D}_w$ of $\mbf{W}$ is defined as the smallest closed subset $\mcl{D}_w$ of $\mbb{R}^{d_z}$ such that $P(\mbf{W} \in \mcl{D}_w) = 1$, and $\mbf{w}_{\min} \in \mcl{D}_w$ minimizes $h(\mbf{w})$.

\begin{lemma}
\label{lemma:prenondecreaseV}
For all $\alpha \in (0,1)$, $V_\alpha(h(\mbf{W}))$ is a nondecreasing function, i.e.,
\begin{align*}
\forall\ 1 > \alpha > \alpha' > 0,\ V_\alpha(h(\mbf{W})) \ge V_{\alpha'}(h(\mbf{W}))\ .
\end{align*}
\end{lemma}

\begin{proof}
Since $\alpha > \alpha'$, for all $\omega \in \mbb{R}$,
\begin{align*}
P(h(\mbf{W}) \le \omega) \ge \alpha \Rightarrow P(h(\mbf{W}) \le \omega) \ge \alpha'\ .
\end{align*}
Therefore, $\{\omega: P(h(\mbf{W}) \le \omega) \ge \alpha\}$ is a subset of $\{\omega: P(h(\mbf{W}) \le \omega) \ge \alpha'\}$. Thus,
\begin{align*}
&\inf \{\omega: P(h(\mbf{W}) \le \omega) \ge \alpha\} \\
&\ge \inf \{\omega: P(h(\mbf{W}) \le \omega) \ge \alpha'\}
\end{align*}
i.e., $V_\alpha(h(\mbf{W})) \ge V_{\alpha'}(h(\mbf{W}))\ .$
\end{proof}

Let 
\begin{align}
\omega_{0^+} \triangleq \lim_{\alpha \rightarrow 0^+} V_\alpha(h(\mbf{W}))\ .
\label{eq:omega0plus}
\end{align}
Then, from Lemma~\ref{lemma:prenondecreaseV}, the following lemma follows.
\begin{lemma}
\label{lemma:nondecreaseV}
For all $\alpha \in (0,1)$, and $\omega_{0^+}$ defined in \eqref{eq:omega0plus}
\begin{align*}
\omega_{0^+} \le V_\alpha(h(\mbf{W}))\ .
\end{align*}
\end{lemma}
We use Lemma~\ref{lemma:nondecreaseV} to prove the following lemma.
\begin{lemma} 
\label{lemma:omegalehw}
For all $\mbf{w} \in \mcl{D}_w$, and $\omega_{0^+}$ defined in \eqref{eq:omega0plus}
\begin{align*}
\omega_{0^+} \le h(\mbf{w})
\end{align*}
which implies that
\begin{align*}
\omega_{0^+} \le h(\mbf{w}_{\min})\ .
\end{align*}
\end{lemma}
\begin{proof}
By contradiction, we assume that there exits $\mbf{w}' \in \mcl{D}_w$ such that $\omega_{0^+} > h(\mbf{w}')$. 
Then, there exists $\epsilon_1 > 0$ such that $\omega_{0^+} > h(\mbf{w}') + \epsilon_1$. 
Consider the pre-image $\mcl{V}$ of the open interval $I_h = (h(\mbf{w}') - \epsilon_1/2, h(\mbf{w}') + \epsilon_1/2$. 
Since $h$ is a continuous function, $\mcl{V}$ is an open set and it contains $\mbf{w}'$ (as $I_h$ contains $h(\mbf{w}')$).
Then, consider the set $\mcl{V} \cap \mcl{D}_w \supset \{\mbf{w}'\} \neq \emptyset$, we prove $P(\mbf{W} \in \mcl{V} \cap \dz) > 0$ by contradiction as follows. 

If $P(\mbf{W} \in \mcl{V} \cap \mcl{D}_w) = 0$ then the closure of $\mcl{D}_w \setminus \mcl{V}$ is a closed set that is smaller than $\mcl{D}_w$ (since $\mcl{V}$ is an open set, $\mcl{D}_w$ is a closed set, and $\mcl{V} \cap \mcl{D}_w$  is not empty) and satisfies $P(\mbf{W} \in \mcl{D}_w \setminus \mcl{V}) = 1$, which contradicts the definition of $\mcl{D}_w$. Thus, $P(\mbf{W} \in \mcl{V} \cap \mcl{D}_w) > 0$.

Therefore, $P(h(\mbf{W}) \in I_h) > 0$. So,
\begin{align*}
&P(h(\mbf{W}) \le \omega_{0^+})\\
&\ge P(h(\mbf{W}) \le h(\mbf{w}') + \epsilon_1 / 2)\\
&\ge P(h(\mbf{W}) \in I_h)\\
&> 0\ .
\end{align*}
Let us consider $\alpha_0 = P(h(\mbf{W}) \le h(\mbf{w}') + \epsilon_1/2) > 0$, the {\var} at $\alpha_0$ is  
\begin{align*}
    V_{\alpha_0}(h(\mbf{W})) &\triangleq \inf\{\omega: P(h(\mbf{W}) \le \omega) \ge \alpha_0\}\\
    &\le h(\mbf{w}') + \epsilon_1/2\\
    &< \omega_{0^+}
\end{align*}
which is a contradiction to Lemma~\ref{lemma:nondecreaseV}. 
\end{proof}
\begin{lemma}
\label{lemma:omegagemin}
For $\omega_{0^+}$ defined in \eqref{eq:omega0plus}
\begin{align}
\omega_{0^+} \ge h(\mbf{w}_{\min})\ .
\end{align}
\end{lemma}
\begin{proof}
By contradiction, we assume that $\omega_{0^+} < h(\mbf{w}_{\min})$. Then there exists $\epsilon_2 > 0$ that $\omega_{0^+} + \epsilon_2 < h(\mbf{w}_{\min})$. 
Since $\omega_{0^+} \triangleq \lim_{\alpha \rightarrow 0^+} V_{\alpha}(h(\mbf{W}))$ so there exits $\alpha_0 > 0$ such that $V_{\alpha_0}(h(\mbf{W})) \in (\omega_{0^+}, \omega_{0^+} + \epsilon_2)$. 
However, 
\begin{align*}
&P(h(\mbf{W}) \le V_{\alpha_0}(h(\mbf{W}))) \\
&\le P(h(\mbf{W}) \le \omega_{0^+} + \epsilon_2 < h(\mbf{w}_{\min}))\\
&= 0
\end{align*} 
which contradicts the fact that $\alpha_0 > 0$.
Therefore, $\omega_{0^+} \ge h(\mbf{w}_{\min})$.
\end{proof}

From \eqref{eq:omega0plus}, Lemma~\ref{lemma:omegalehw} and Lemma~\ref{lemma:omegagemin}, 
\begin{align*}
\lim_{\alpha \rightarrow 0^+} V_\alpha(h(\mbf{W})) = h(\mbf{w}_{\min})
\end{align*}
which directly leads to the result in Corollary~\ref{corollary:alpha0lv} for a continuous function $f(\mbf{x},\mbf{z})$ over $\mbf{z} \in \dz$.
While $\mbf{Z}$ can follow any probability distribution defined on the support $\dz$, we can choose the distribution of $\mbf{Z}$ as a uniform distribution over $\dz$.

\subsection{Corollary~\ref{corollary:stableopt}}

From Theorem~\ref{theorem:0plus}, $\dz$ is a closed subset of $\mbb{R}^{d_z}$, and $u_{t-1}(\mbf{x},\mbf{z})$, $l_{t-1}(\mbf{x},\mbf{z})$ are continuous functions over $\mbf{z} \in \dz$, it follows that the selected $\mbf{x}_t$ by both {\stableopt} (in \eqref{eq:stableopt}) and {\vy} are the same. Furthermore, 
\begin{align*}
\mcl{Z}_l^\le &\triangleq \{ \mbf{z} \in \dz: l_{t-1}(\mbf{x},\mbf{z}) \le V_\alpha(l_{t-1}(\mbf{x},\mbf{Z})) \}\\
	&=  \{ \mbf{z} \in \dz: l_{t-1}(\mbf{x},\mbf{z}) \le \min_{\mbf{z}' \in \dz}l_{t-1}(\mbf{x},\mbf{z}') \}\\
	&=  \{ \mbf{z} \in \dz: l_{t-1}(\mbf{x},\mbf{z}) = \min_{\mbf{z}' \in \dz}l_{t-1}(\mbf{x},\mbf{z}') \}\ ,\\
\mcl{Z}_u^\ge &\triangleq \{ \mbf{z} \in \dz: u_{t-1}(\mbf{x},\mbf{z}) \ge V_\alpha(u_{t-1}(\mbf{x},\mbf{Z})) \}\\
	&= \{ \mbf{z} \in \dz: u_{t-1}(\mbf{x},\mbf{z}) \ge \min_{\mbf{z}' \in \dz} u_{t-1}(\mbf{x},\mbf{z}') \}\\
	&= \dz\ .
\end{align*}

Therefore, the set of lacing values is $\mcl{Z}_l^\le \cap \mcl{Z}_u^\ge = \mcl{Z}_l^\le = \{ \mbf{z} \in \dz: l_{t-1}(\mbf{x},\mbf{z}) = \min_{\mbf{z}' \in \dz}l_{t-1}(\mbf{x},\mbf{z}') \}$ any of which is also the selected $\mbf{z}_t$ in \eqref{eq:stableopt} by {\stableopt}. Thus, the selected $\mbf{z}_t$ by both {\stableopt} and {\vy} are the same.

\section{Local Neural Surrogate Optimization}
\label{app:lnso}

The \emph{local neural surrogate optimization} (LNSO) to maximize a {\var} $V_\alpha(h(\mbf{x},\mbf{Z}))$ is described in Algorithm~\ref{alg:lnso}. The algorithm can be summarized  as follows:
\begin{itemize}
\item Whenever the current updated $\mbf{x}^{(i)}$ is not in $\mcl{B}(\mbf{x}_c,r)$ (line 4), the center $\mbf{x}_c$ of the ball $\mcl{B}$ is updated to be $\mbf{x}^{(i)}$ (line 6) and the surrogate function $g(\mbf{x},\bm{\theta})$ is re-trained (lines 7-12).

\item The surrogate function $g(\mbf{x},\bm{\theta})$ is (re-)trained to estimate $V_\alpha(h(\mbf{x},\mbf{Z}))$ well for all $\mbf{x} \in \mcl{B}(\mbf{x}_c, r)$ (lines 7-12) with stochastic gradient descent by minimizing the following loss function given random mini-batches $
\mcl{Z}$ of $\mbf{Z}$ (line 8) and $\mcl{X}$ of $\mbf{x} \in \mcl{B}(\mbf{x}_c,r)$ (line 9):
\begin{align}
\mcl{L}_g(\mcl{X},\mcl{Z}) \triangleq \frac{1}{|\mcl{X}| |\mcl{Z}|} \sum_{\mbf{x} \in \mcl{X},\mbf{z} \in \mcl{Z}} [\rho_\alpha(h(\mbf{x},\mbf{z}) - g(\mbf{x};\bm{\theta}))]
\label{eq:lossg}
\end{align}
where $\rho_{\alpha}$ is the pinball function in Sec.~\ref{sec:continuousz}.

\item Instead of directly maximizing $V_\alpha(h(\mbf{x},\mbf{Z}))$ whose gradient w.r.t $\mbf{x}$ is unavailable, we find $\mbf{x}$ that maximizes the surrogate function $g(\mbf{x},\bm{\theta}_s)$ (line 14) where $\bm{\theta}_s$ is the parameters trained in lines 7-12.

\end{itemize}
\begin{algorithm}[tb]
   \caption{LNSO of $V_\alpha(h(\mbf{x},\mbf{Z}))$}
\begin{algorithmic}[1]
   \STATE {\bfseries Input:} target function $h$; domain $\dx$; initializer $\mbf{x}^{(0)}$; $\alpha$; a generator of $\mbf{Z}$ samples \texttt{gen\_Z}; radius $r$; no. of training iterations $t_v$, $t_g$; optimization stepsizes $\gamma_x$, $\gamma_g$

   \STATE Randomly initialize $\bm{\theta}_s$
   \FOR{$i=1,2,\dots, t_v$}
      \IF{$i=1$ or $\Vert \mbf{x}^{(i)} - \mbf{x}_c \Vert \ge \delta_x$}
      \STATE Initialize $\bm{\theta}^{(0)} = \bm{\theta}_s$
      \STATE Update the center of $\mcl{B}$: $\mbf{x}_c = \mbf{x}^{(i)}$
      \FOR{$j=1,2,\dots, t_g$}
      \STATE Draw $n_z$ samples of $\mbf{Z}$: $\mcl{Z} =\texttt{gen\_Z}(n_z)$.
      \STATE Draw a set $\mcl{X}$ of $n_x$ uniformly distributed samples in $\mcl{B}(\mbf{x}_c,r)$.
      \STATE Update 
        $\bm{\theta}^{(j)} = \bm{\theta}^{(j-1)} - \gamma_g \frac{\text{d} \mcl{L}_g(\mcl{X},\mcl{Z})}{\text{d} \bm{\theta}}\Big|_{\bm{\theta} = \bm{\theta}^{(j-1)}}$ where $\mcl{L}_g(\mcl{X},\mcl{Z})$ is defined in \eqref{eq:lossg}.
      \ENDFOR
      \STATE $\bm{\theta}_s = \bm{\theta}_{t_g}$
      \ENDIF
   \STATE Update $\mbf{x}^{(i)} = \mbf{x}^{(i-1)} + \gamma_x \frac{\text{d}g(\mbf{x};\bm{\theta}_s)}{ \text{d}\mbf{x}}\big|_{\mbf{x} = \mbf{x}^{(i-1)}}$.
   \STATE Project $\mbf{x}^{(i)}$ into $\dx$.
   \ENDFOR
   \STATE Return $\mbf{x}^{(t_v)}$
\end{algorithmic}
\label{alg:lnso}
\end{algorithm}

\section{Experimental Details}
\label{app:experiment}

Regarding the construction of $\dz$ in optimizing the synthetic benchmark functions, the discrete $\dz$ is selected as equi-distanct points (e.g., by dividing $[0,1]^{d_z}$ into a grid).
The probability mass of $\mbf{Z}$ is defined as $P(\mbf{Z} = \mbf{z}) \propto \exp(-(\mbf{z} - 0.5)^2 / 0.1^2)$ (the subtraction $\mbf{z} - 0.5$ is elementwise). 
The continuous $\mbf{Z}$ follows a $2$-standard-deviation truncated independent Gaussian distribution with the mean of $0.5$ and standard deviation $0.125$. It is noted that when $\dz$ is discrete, there is a large region of $\mbf{Z}$ with low probability $P(\mbf{Z})$ in experiments with synthetic benchmark functions. This is to highlight the advantage of {\vy} Prob in exploiting $P(\mbf{Z})$ compared with {\vy} Unif. In the robot pushing experiment, the region of $\mbf{Z}$ with low probability is smaller than that in the experiments with synthetic benchmark functions (e.g., Hartmann-$(1,2)$), which is illustrated in Fig.~\ref{fig:zprob}. Therefore, the gap in the performance between {\vy} Unif and {\vy} Prob is smaller in the robot pushing pushing experiment (Fig.~\ref{fig:portrobot}b) than that in the experiment with Hartmann-$(1,2)$ (Fig.~\ref{fig:synfinite}c).
\begin{figure}
    \centering
    \includegraphics[width=0.45\textwidth]{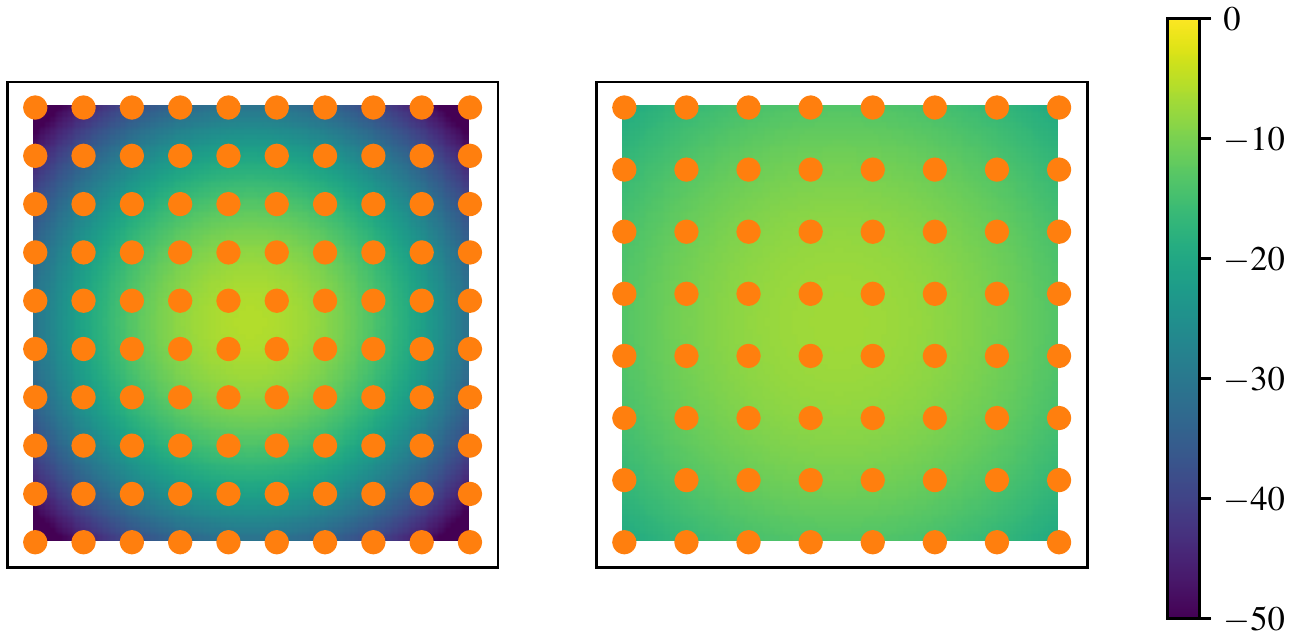}
    \caption{Plots of the log values of the un-normalized probabilities of the discrete $\mbf{Z}$ for the Hartmann-$(1,2)$ in the left plot and Robot pushing $(3,2)$ in the right plot. The orange dots show the realizations of the discrete $\mbf{Z}$.}
    \label{fig:zprob}
\end{figure}

When the closed-form expression of the objective function is known (e.g., synthetic benchmark functions) in the evaluation of the performance metric, the maximum value $\max_{\mbf{x} \in \dx} V_\alpha(f(\mbf{x},\mbf{Z}))$ can be evaluated accurately. On the other hand, when the closed-form expression of the objective function is unknown even in the evaluation of the performance metric (e.g., the simulated robot pushing experiment), the maximum value $\max_{\mbf{x} \in \dx} V_\alpha(f(\mbf{x},\mbf{Z}))$ is estimated by $\max_{\mbf{x} \in \mcl{D}_T} V_\alpha(f(\mbf{x},\mbf{Z})) + 0.01$ where $\mcl{D}_T$ are input queries in the experiments with both {\vy} and {\kg}. The addition of $0.01$ is to avoid $-\infty$ value in plots of the log values of the performance metric.

The sizes of the initial observations $\mcl{D}_0$ are $3$ for the Branin-Hoo and Goldstein-Price functions; $10$ for the Hartmann-3D function; $20$ for the portfolio optimization problem; and $30$ for the simulated robot pushing task. 
The initial observations are randomly sampled for different random repetitions of the experiments, but they are the same between the same iterations in {\vy} and {\kg}. 

The hyperparameters of GP (i.e., the length-scales and signal variance of the SE kernel) and the noise variance $\sigma_n^2$ are estimated by maximum likelihood estimation \cite{rasmussen06} every $3$ iterations of BO. We set a lower bound of $0.0001$ for the noise variance $\sigma_n^2$ to avoid numerical errors.

To show the advantage of LNSO, we set the number of samples of $\mbf{W}$ to be $10$ for both {\vy} and {\kg}. The number of samples of $\mbf{x}$, i.e., $|\mcl{X}|$, in LNSO (line 9 of Algorithm~\ref{alg:lnso}) is $50$. The radius $r$ of the local region $\mcl{B}$ is set to be a small value of $0.1$ such that a small neural network works well: $2$ hidden layers with $30$ hidden neurons at each layer; the activation functions of the hidden layers and the output layer are sigmoid and linear functions, respectively.

Since the theoretical value of $\beta_t$ is often considered as excessively conservative \cite{bogunovic16,srinivas10ucb,bogunovic2018adversarially}. We set $\beta_t = 2\log(t^2 \pi^2/0.6)$ in our experiments while $\beta_t$ can be tuned to achieved better exploration-exploitation trade-off \cite{srinivas10ucb} or multiple values of $\beta_t$ can be used in a batch mode \cite{torossian2020bayesian}.

\end{document}